\PassOptionsToPackage{sort}{natbib}
\documentclass[preprint,11pt,english]{imsart}
\usepackage{subfigure}
\usepackage[utf8]{inputenc}
\usepackage[T1]{fontenc}
\usepackage{appendix}
\usepackage{fullpage}
\usepackage{tabularx}
\usepackage{amsthm,amsmath,amsfonts,amssymb}
\usepackage{booktabs}
\usepackage{multirow}
\usepackage{enumitem}
\usepackage{mathtools}
\usepackage{esint}
\usepackage{overpic}
\usepackage{xcolor}
\usepackage{soul}
\usepackage{nicefrac}
\usepackage[normalem]{ulem}

\usepackage[colorlinks = true, pdfstartview = FitV, linkcolor = blue, citecolor = blue, urlcolor = blue, bookmarks = false]{hyperref}

\newcommand{\iidsim}{\overset{\text{iid}}{\sim}}

\newcommand{\dd}{\mathrm{d}}
\newcommand{\tr}{\mathrm{tr}}

\newcommand{\IIC}{\mathrm{IIC}}
\newcommand{\cond}{\mathrm{cond}}
\newcommand{\cov}{\mathrm{Cov}}

\DeclareMathOperator*{\argmin}{arg\,min}
\DeclareMathOperator*{\argmax}{arg\,max}

\newtheorem{theorem}{Theorem}
\newtheorem{lemma}{Lemma}

\newtheorem{proposition}{Proposition}

\theoremstyle{definition}
\newtheorem{definition}{Definition}

\newtheorem*{example*}{Running Example}
\newtheorem{assumption}{Assumption}
\newtheorem*{condition}{Conditions}
\newtheorem*{assumption*}{Assumption}
\newtheorem{remark}{Remark}
\newtheorem*{remark*}{Remark}

\usepackage{relsize}

\begin{document}
\begin{frontmatter}
\title{The Interpolating Information Criterion for \\ 
Overparameterized Models 
}

\runtitle{Interpolating Information Criterion}
\runauthor{Hodgkinson et al.}

\begin{aug}

\author[A]{\fnms{Liam} \snm{Hodgkinson}\ead[label=e1]{lhodgkinson@unimelb.edu.au}},
\author[B]{\fnms{Chris}
\snm{van der Heide}\ead[label=e2]{chris.vdh@gmail.com}},
\author[C]{\fnms{Robert} \snm{Salomone}\ead[label=e3]{r.salomone@unsw.edu.au}},
\author[D]{\fnms{Fred} \snm{Roosta}\ead[label=e4]{fred.roosta@uq.edu.au}}
\and
\author[E]{\fnms{Michael W.} \snm{Mahoney}\ead[label=e5]{mahoney@stat.berkeley.edu}}

\address[A]{School of Mathematics and Statistics, University of Melbourne. \\ \printead{e1}}
\address[B]{Department of Electrical and Electronic Engineering, University of Melbourne. \\ \printead{e2}}
\address[C]{School of Mathematics and Statistics, UNSW Sydney. \\ \printead{e3}}
\address[D]{CIRES and School of Mathematics and Physics, University of Queensland. \\ 
\printead{e4}}
\address[E]{ICSI, LBNL, and Department of Statistics, University of California, Berkeley.\\
\printead{e5}}

\end{aug}

\begin{abstract}
The problem of model selection is considered for the setting of interpolating estimators, where the number of model parameters exceeds the size of the dataset. %
Classical information criteria typically consider the large-data %
limit, penalizing model size. However, these criteria are not appropriate in modern settings where overparameterized models tend to perform well. 
For any overparameterized model, we show that there exists a dual underparameterized model that possesses the same marginal likelihood, thus establishing a form of \emph{Bayesian duality}.
This enables more classical methods to be used in the overparameterized setting, revealing the \emph{Interpolating Information Criterion}, a measure of model quality that naturally incorporates the choice of prior into the model selection. 
Our new information criterion accounts for prior misspecification, geometric and spectral properties of the model, and is numerically consistent with known empirical and theoretical behavior in this regime.

\end{abstract}

\end{frontmatter}

\maketitle

\section{Introduction}
The task of {\em model selection}, that is, determining which of a prescribed set of candidate models is most suitable in some sense, is fundamental in statistical learning. 
Some of the most important model selection tools in a modern statistician's toolbox are \emph{cross validation} and \emph{information criteria} \cite{konishi2008information}, with the latter being the focus of this work. Assigning each model a real number, where lower numbers are to be preferred, information criteria typically trade off model performance and complexity, providing practitioners with a quantitative framework to apply 
Occam's razor~\cite{mackay1992bayesian}. 
Two of the most commonly used information criteria are the \emph{Akaike information criterion} (AIC) \cite{akaike1998information} and the \emph{Bayesian information criterion} (BIC) \cite{schwarz1978estimating}.
The derivation of the BIC, in particular, considers the {\em Bayesian} formulation of an information criterion. This involves prescribing, to a given model, a {\em prior distribution} on the set of admissible parameters, which has density $\pi$. The normalizing constant of the posterior distribution is the {\em marginal likelihood} of the data under the chosen model and expresses a given model's preference for the data. 
Deep connections between marginal likelihood, cross validation \cite{fong2020marginal}, and PAC-Bayes generalization risk bounds \cite{germain2016pac}, make the negative log-marginal likelihood (often called the \emph{Bayes free energy}) an obvious choice as an information criterion.
When all candidate models are weighted equally \emph{a priori}, classical Bayesian model selection amounts to choosing the model with the largest marginal likelihood.

Since exact computation of marginal likelihoods is often intractable in practice~\cite{bos2002comparison}, the BIC approximates the log-marginal likelihood in the large data regime (as the size of the dataset $N \to \infty$).
However, the underlying approximation makes use of Laplace's method, requiring invertibility of the Hessian of the log-likelihood, and consequently it fails in the overparameterized setting,\footnote{A key distinction is that overparameterized models are defined to be those that have more parameters than data, and not necessarily those with \emph{more parameters than are ``necessary''}. The latter judgment implicitly requires an information criterion to quantify an optimal (``necessary'') number of parameters, typically the AIC or BIC. Indeed, a consequence of this work is that such a quantity may exceed the number of data points or be infinite. } where the number of parameters $d$ in the model exceeds the number of data points $N$ \cite{wei2022deep}. 
Alternatives suitable for singular models, where the Hessian is not invertible, have also been proposed \cite{drton2017bayesian, watanabe2013widely} leveraging the \emph{singular learning theory} pioneered by Watanabe \cite{watanabe2007almost,watanabe2009algebraic}. However, these too rely on the large data limit with fixed model size, which fundamentally cannot characterize the $d\gg N$ behavior seen in modern machine learning. Significant overparameterization admits classes of large models (e.g., deep neural networks) that are able to \emph{interpolate}: the predictive function can precisely match all training data to its labels.
Contrary to widely held belief,\footnote{The common belief is clearly stated in \cite[pg. 37]{hastie2009elements}: ``interpolating fits... [are] unlikely to predict future data well at all.'' Several of these same authors would later go on to publish prominent works highlighting ``surprising'' cases where this rule of thumb fails (double descent) \cite{hastie2022surprises}. This is discussed in further detail in \S5.2.} even when training to near-zero loss, these models can still attain excellent generalization performance \cite{hastie2022surprises,zhang2021understanding}.

Assuming mild regularity, the parameters obtained by training in this manner lie near some embedded submanifold $\mathcal{M}$ of the parameter space, corresponding to the zero level-set of the loss function. 
In practice, a convergent method that trains a model to zero-loss will yield a \emph{single} parameter choice $\theta^\star \in \mathcal{M}$. While any parameter in $\mathcal{M}$ is viable, the training procedure forces the choice of a unique element of $\mathcal{M}$. Such a procedure is tantamount to {\em implicitly} imposing a regularization function which was used to select the chosen $\theta^\star \in \mathcal{M}$ as the solution to a constrained optimization problem on $\mathcal{M}$. Minimizing a regularizer $R$ over $\mathcal{M}$ is equivalent to maximizing a corresponding prior density $\pi$ over $\mathcal{M}$ where the negative log-prior is proportional to $R$. In this scenario, the prior is defined implicitly through the training procedure \cite{neyshabur2017implicit}, playing the important role of identification by placing higher probabilities on certain regions of $\mathcal{M}$. 

In light of the above, we construct an information criterion by examining the marginal likelihood under a regime comprised of \emph{two} limits taken sequentially: the first maximizes the likelihood (achieving zero loss), and the second maximizes the prior (concentrating around the interpolator $\theta^\star$). This strategy provides a new approximation which replaces the underparameterized, large data regime ($N \to \infty$), with an overparameterized, interpolating regime. %

The primary contribution of this work is the following interpolating information criterion (IIC) for overparameterized models and the theoretical tools required for its derivation:
\begin{equation}
\label{eq:IIC}
\IIC=\underset{\text{regularization}}{\underbrace{\log\log\frac{\pi(\theta_0)}{\pi(\theta^{\star})}}}+\underset{\text{sharpness}}{\underbrace{\vphantom{\int}\frac1N\log\det  \big( DF(\theta^\star) DF(\theta^\star)^\top \big)}}+ \underset{\text{curvature}}{\underbrace{\vphantom{\int}\frac1N\log \mathcal{K}_{\mathcal{M}}^\pi(\theta^\star, \theta_0)}}- \underset{\text{correction}}{\underbrace{\vphantom{\int}\log N}}, %
\end{equation} 
where the explicit expression for the relative curvature term $\mathcal{K}_{\mathcal{M}}^\pi$ is given in \eqref{eq:curve}. 
Here, $\theta_0$ and $\theta^\star$ are respectively, the global maximizers of the prior over an underlying parameter space $\Theta$ and on $\mathcal{M}$, $F(\theta) = (f(x_i,\theta))_{i=1}^n$ for $f$ a parameterisation of the model class and $x_1,\dots,x_n \in \mathcal{X}$ input data, and $DF(\theta^{\star})$ is the Jacobian of $F$ at $\theta^\star$. Here, $n$ is the number of points in the dataset, each with $m$ labels, and $N = mn$. %
The IIC contains three key terms as well as a correction for data size, depicted in Figure~\ref{fig:intuition_terms}: the first term penalizes prior misspecification / the appropriateness of the corresponding regularizer; the second encourages local smoothness of the predictive function; and the third compares the relative volume of the prior %
at its peak on $\mathcal{M}$ and $\Theta$, accounting for curvature of the submanifold $\mathcal{M}$. The precise nature of these terms is discussed further in Section \ref{sec:Interpretation}. %
\begin{figure}[t]
\begin{tabularx}{\textwidth}{XXX}
\centering \textbf{Regularization} & 
\centering \textbf{Sharpness} & \centering \textbf{Curvature}
\end{tabularx}
\begin{tabularx}{\textwidth}{XXX}
\centering $\log\log \frac{\pi(\theta_0)}{\pi(\theta^\star)}$ & 
\centering $\log \det(DF(\theta^\star)DF(\theta^\star)^\top)$ & \centering 
$\log \mathcal{K}_{\mathcal{M}}^\pi(\theta^\star,\theta_0)$
\end{tabularx}
\vspace{.3cm}

\begin{overpic}[width=0.9\textwidth]{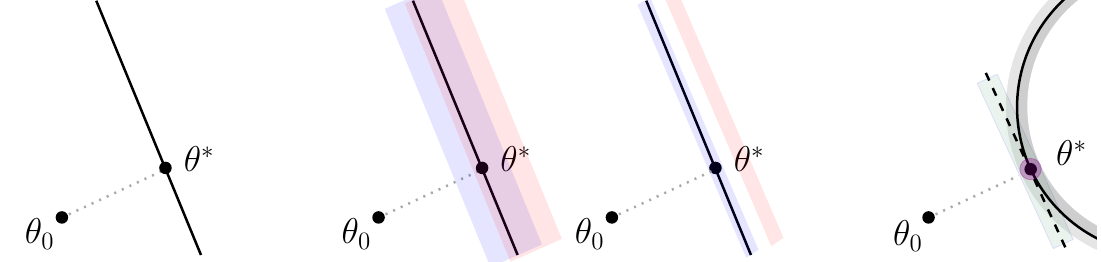}
\put(30,20){train}
\put(44,22){test}
\put(54,20){train}
\put(63,22){test}
\put(52,13){\large vs.}
\end{overpic}

\caption{\label{fig:intuition_terms}
Visualizing the three primary terms of the IIC. 
(Left) The regularization term measures closeness of $\theta_0$ to the interpolating manifold. 
(Center) The sharpness term encourages flatter vs. sharper minima in the loss, as this suggests the region of small training loss (blue) overlaps more regions of small test loss (red)---see \cite{keskar2016large} for a similar visualization in one-dimension. 
(Right) The curvature term penalizes regions where the vector normal to the prior is less stable than around its global minima, as this suggests more of the neighboring region along the manifold $\mathcal{M}$ (grey) falls outside the region of high prior probability (green).
}
\end{figure}
\begin{figure}[tb]
    \centering
    \includegraphics[width=0.7\textwidth]{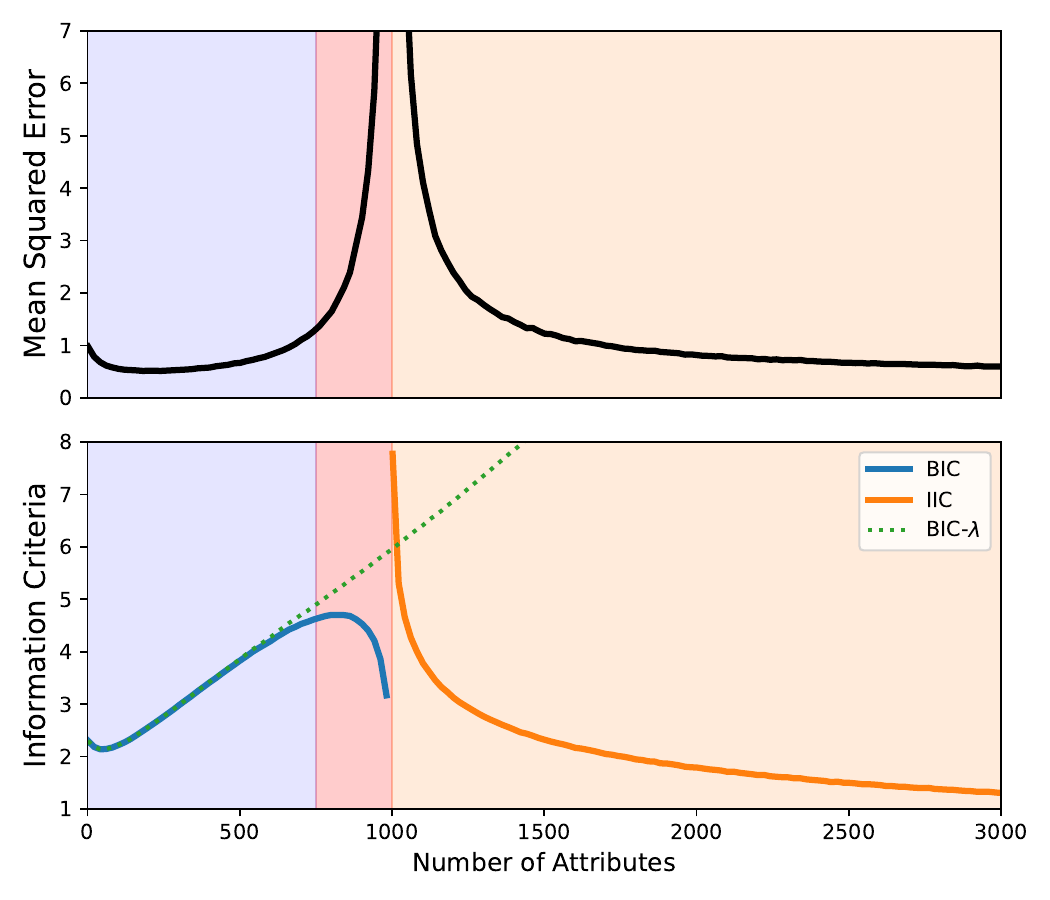}
    \caption{Mean squared error (MSE; top) vs. (bottom) classical BIC (\ref{eq:BICLinearReg}), our novel IIC (\ref{eq:IICLinearReg}), and the BIC for ridge regression with ridge parameter $\lambda=0.1$ \emph{(BIC-$\lambda$)}, for random Fourier features with varying number of attributes. Measures are averaged over $100$ iterations applied to random subsamples of $n=1000$ input-output pairs from the \texttt{MNIST} dataset \cite{lecun1998gradient}. The underparameterized, critical (where BIC fails), and overparameterized regimes are highlighted blue, red, and yellow, respectively. \textbf{Excluding the critical region, the combined BIC and IIC curve exhibits double-descent.}\label{fig:DDRFF}}
\end{figure}

An important secondary contribution of this work, from which the IIC is derived, is the characterization of a strong form of \emph{Bayesian duality}.\footnote{Note that this is not the first investigation of forms of duality arising in the Bayesian setting; for example, \cite{khan2021bayesian,mollenhoff2022sam} explore other notions. The formulation herein is distinct in its consideration of a distributional version of augmented Lagrangian duality, the latter having its origins in optimization theory.} 
More specifically, similar to \cite{de2021quantitative}, we leverage tools from geometric measure theory to show that an overparameterized model possesses an underparameterized dual model over data space with equal marginal likelihood. This dual model can be shown to be smooth in many cases, satisfying the conditions for an analogue of the Laplace approximation. Applying this approximation to the dual model reveals the IIC. Since the IIC is constructed by concentration onto the interpolating submanifold, it is non-asymptotic in both $d$ and $N$. Bayesian duality explains the ``surprising'' double descent curve, famously raised in \cite{belkin2019reconciling} (see also \cite{hastie2022surprises,loog2020brief}), which can be pieced together from a combination of the classical BIC and our introduced IIC (Figure \ref{fig:DDRFF}). For more details, we refer to the discussion in Section \ref{sec:DD}. 

The remainder of the paper is organized as follows. 
Section~\ref{sec:Preliminaries} provides the requisite background surrounding interpolating solutions to model fitting problems, appealing to duality from the optimization point of view. 
Section~\ref{sec:MargLikeDuality} introduces the relevant statistical frameworks, and establishes a key result that characterizes Bayesian duality (Proposition \ref{thm:DR}). Section~\ref{sed:IIC} provides a version of Laplace's method on manifolds which may be of independent interest (Lemma~\ref{prop:LaplaceManifold}), and allows for the derivation of the IIC (Theorem~\ref{thm:IIC}). Finally, Section~\ref{sec:Wider} discusses related work, the significance of the IIC, and the individual terms in \eqref{eq:IIC}, as they relate to existing theory and practice in the deep learning literature. A suite of numerical experiments is presented in Section~\ref{sec:Numerics}, highlighting correlations between IIC and other measurements of model quality in linear and gamma regression (Section~\ref{sec:LinGamReg}); predicting the poor performance of monomial vs. Chebyshev polynomial interpolants (Section~\ref{sec:Poly}); and highlighting the importance of the regularizer for predicting performance in diagonal linear neural networks (Section~\ref{sec:DLNN}). For the reader's convenience, key supporting results are provided in the appendices in language that is consistent with our work.

\section{Preliminaries}\label{sec:Preliminaries}

Consider a parameterized class of predictors $\{f(\cdot,\theta)\}_{\theta\in\Theta}$ %
where $f:\mathcal{X}~\times~\Theta~\to~\mathcal{Y}$, for $\mathcal{Y} \subseteq \mathbb{R}^m$, %
which we take to be $\mathcal{C}^\infty$-smooth on $\Theta \subseteq \mathbb{R}^d$. 
For a fixed dataset of input-output pairs $\mathcal{D} = (x_{i},y_{i})_{i=1}^{n} \subset \mathcal{X} \times \mathcal{Y}$, the corresponding \emph{regression problem} seeks to find parameters for the most suitable predictor. 
For $\ell:\mathcal{Y}\times \mathcal{Y} \to [0,\infty)$ a smooth loss function measuring the accuracy of a prediction for a single input-output pair, the corresponding \emph{Gibbs likelihood} is $$p_\gamma(y_i \vert \theta, x_i) \propto \exp\left(-\frac1\gamma \ell(f(x_i,\theta),y_i)\right),$$ where the temperature $\gamma > 0$ is arbitrary. 
For brevity, we write $L(y,y') = \sum_{i=1}^n \ell(y_i, y_i')$, where $y = (y_i)_{i=1}^n \in \mathbb{R}^{n\times m}$ and $F:\Theta \to \mathbb{R}^{n\times m}$ to denote $F(\theta) = (f_j(x_i,\theta))_{i,j=1}^{n,m}$. %
In the sequel we will assume that for any $y,y'$, (a) $\ell(y,y') \geq 0$; and (b) $\ell(y,y') = 0$ if and only if $y = y'$. 
Note that (a) is equivalent to boundedness from below, while (b) uniquely characterises minima of $\ell$. Such loss functions are ubiquitous in regression tasks. In particular, we have $L^2$ regression problems in mind (where $\ell(y,y') = \|y-y'\|^2$), and aim to keep the model class $f$ as general as possible. Let $DF:\mathbb{R}^d \to \mathbb{R}^{mn \times d}$ denote the Jacobian of the vectorization of $F$ and write $J(\theta) \coloneqq DF(\theta) DF(\theta)^\top \, :\, \mathbb{R}^d \to \mathbb{R}^{mn \times mn}$. 

The \emph{marginal likelihood} (also \emph{model evidence} or \emph{partition function}) is the normalizing constant of the posterior %
\[
\mathcal{Z}_{n,\gamma} = \int_{\Theta} \pi(\theta) \prod_{i=1}^n p_\gamma(y_i \vert \theta,x_i) \dd \theta = \int_\Theta c_{n,\gamma}(F(\theta)) e^{-\frac1\gamma L(F(\theta), y)} \pi(\theta) \dd \theta,
\]
where $c_{n,\gamma}(z)^{-1} = \prod_{i=1}^n\int_{\mathcal{Y}} e^{-\frac1\gamma \ell(z_i,y')} \dd y'$ is assumed finite for $z = (z_i)_{i=1}^n \in \mathbb{R}^{n\times m}$, and is the main object of study in this work. 
The marginal likelihood is the probability $p(\mathcal{D})$ of the dataset under the prescribed likelihood-prior pair, and is commonly used as a measure of model quality. Referring to the quantity $\mathcal{F}_{n,\gamma} = -\log \mathcal{Z}_{n,\gamma}$ as the \emph{Bayes free energy}, maximizing model quality (as represented by the marginal likelihood) is equivalent to the \emph{free energy principle} in statistical mechanics.

\subsection{Interpolators}
\label{sec:Interpolators}

Point estimators for regression models are typically maximum likelihood estimators (MLEs); for the Gibbs likelihood, these are
\begin{equation}
\tag{MLE}
\label{eq:MLE}
\theta^\star \in \mathcal{M} \coloneqq \argmax_{\theta \in \Theta} \prod_{i=1}^n p_\gamma(y \vert \theta,x_i) = \argmin_{\theta \in \Theta} \sum_{i=1}^n \ell(f(x_i, \theta), y_i).
\end{equation}
Regardless of the form of $f$, since $\ell$ is nonnegative, any $\theta$ such that $f(x_i,\theta) = y_i$ for each $i=1,\dots,n$ must necessarily be in $\mathcal{M}$. 
These are \emph{interpolators}: 
point estimators that achieve zero loss by perfectly fitting the dataset $\mathcal{D}$. 
We assume that zero loss estimators are achieved, as is the case for many tasks in machine learning. This corresponds to the non-degeneracy assumption that the set of MLEs is non-empty:%
\[
\mathcal{M} = \{\theta\in\Theta\,:\,f(x_i,\theta)=y_i\text{ for all }i=1,\dots,n\} \neq \emptyset.
\]
In the overparameterized setting ($d > mn$), $\mathcal{M}$ is often not only of infinite cardinality, but a submanifold of positive dimension, resulting in an almost canonical lack of identifiability. 
To uniquely identify an estimator, a \emph{regularizer} $R$ is introduced either explicitly \cite{candes2007dantzig,Tibs96} or implicitly via the training procedure \cite{gidel2019implicit,GWBNS17,MO11-implementing,smith2020origin}. We now seek $\theta^\star$ to minimize $R$ on $\mathcal{M}$, leading to the following definition.

\begin{definition}
An \emph{interpolator} %
is a point estimator $\theta^\star$ of the form
\begin{equation}
\tag{INT}
\label{eq:INT}
\theta^\star \in \argmin_{\theta \in \Theta} R(\theta)\quad\mbox{subject to}\quad \theta \in \mathcal{M}.
\end{equation}
\end{definition}
While there is no guarantee that such a minimizer will be unique for arbitrary $f$ and $R$, for each $\theta^\star \in \mathcal{M}$, there exists an $R$ which uniquely identifies $\theta^\star$. This can be seen by taking $R(\theta)~=~\|\theta~-~\theta^\star\|^2$, for example. 
By considering these equality constraints, we capture a wide class of problems where we regress against the outcomes directly. 
However, in classification tasks constructed via multinomial likelihood models for which maximizing the log-likelihood is equivalent to minimizing the cross-entropy loss, interpolators are naturally defined in terms of \emph{inequality constraints}. 
Nevertheless, one can exploit the IIC as an approximation in the classification setting by using a loss such as a suitable adjustment of the Brier score in a generalized Bayesian framework \cite{bissiri2016general}, which typically achieves comparable results in practice \cite{HB21}. 

\begin{example*}[\textsc{Moore--Penrose Pseudoinverse}]
Consider the setting of least-squares linear regression: letting $\mathcal{X} = \Theta = \mathbb{R}^d$, $\mathcal{Y} = \mathbb{R}$, $f(x,\theta) = x\cdot\theta$, and $\ell(y,y') = (y-y')^2$. Writing $X = (x_{ij})_{i=1,j=1}^{n,d} \in \mathbb{R}^{n\times d}$ and $y = (y_i)_{i=1}^n \in \mathbb{R}^n$, if $y \in \text{range}(X)$, then
\begin{equation}
\label{eq:LinSubspace}
\mathcal{M} = \{X^+ y + (I - X^+ X) w : w \in \mathbb{R}^d\},
\end{equation}
where $X^+$ is the \emph{Moore-Penrose pseudoinverse} \cite[\S7.3]{horn2012matrix}, which corresponds to the interpolator with $R(\theta) = \|\theta\|^2$:
\begin{align}\label{eq:ridgeless}
\theta^\star = X^+ y = \argmin_{\theta \in \mathbb{R}^d} \|\theta\|^2\quad\mbox{subject to}\quad x_i \cdot \theta = y_i\text{ for all }i=1,\dots,n.
\end{align}
\end{example*}
Regularized regression has a Bayesian interpretation if we specify a prior density $\pi$, set $R(\theta)=-\log\pi(\theta)$ as the regularizer, and consider the \emph{maximum a posteriori} (MAP) estimator
\begin{equation}
\tag{MAP}
\label{eq:MAP}
\theta^\star \in \argmax_{\theta \in \Theta} \pi(\theta) \prod_{i=1}^n p_\gamma(y\vert \theta,x_i)  = \argmin_{\theta \in \Theta} R(\theta) + \frac{1}{\gamma}L(F(\theta),y).
\end{equation}
In the linear regression setting, \eqref{eq:MAP} can be viewed as a soft-constrained relaxation of \eqref{eq:ridgeless}, as sending $\gamma\to0$ also yields the Moore--Penrose pseudoinverse. This also holds more generally when $R$ is bounded from below, as in Lemma \ref{lem:MAPLimit}. This asymptotic connection between \eqref{eq:MAP} and \eqref{eq:INT} motivates us to consider the limit of the marginal likelihood $\mathcal{Z}_{n,\gamma}$ as $\gamma \to 0^+$ in Section \ref{sec:MargLikeDuality}. %
\begin{lemma}
\label{lem:MAPLimit}
Assume $R$ is bounded from below on $\Theta$. Any limit of a sequence of solutions $\theta_\gamma$ to \eqref{eq:MAP} as $\gamma \to 0^+$ is a solution to \eqref{eq:INT}. 
\end{lemma}
\begin{proof}
Without loss of generality, assume $R$ is non-negative. Let $\theta^\star$ be a solution to \eqref{eq:INT}. Since $R(\theta_{\gamma}) + L(F(\theta_{\gamma}),y)/\gamma \leq R(\theta^{\star})$, this implies $0 \leq L(F(\theta_{\gamma}),y) \leq \gamma R(\theta^{\star})$, and hence $L(F(\theta_{\gamma}),y) \to 0$. %
Therefore, any limit point $\bar{\theta}$ of $\theta_\gamma$ is in $\mathcal{M}$ and satisfies $R(\bar{\theta}) \leq R(\theta^\star)$.
\end{proof}

It is important to highlight 
the subtle point 
that despite Lemma \ref{lem:MAPLimit}, 
and contrary to popular belief, 
(\ref{eq:MAP}) is not a valid ordinary Lagrangian for \eqref{eq:INT}. 
In particular, they are not necessarily equivalent %
for any fixed $\gamma$, and so need not exhibit even qualitatively similar behavior. 
To see this, note the problem (\ref{eq:INT}) is equivalent to
\begin{equation}
\label{eq:HardINT}
\theta^\star \in \argmin_{\theta \in \Theta} R(\theta) \quad \text{subject to} \quad L(F(\theta),y) = 0,
\end{equation}
but since every $\theta\in\mathcal{M}$ is a critical point of the loss, 
unless $\nabla R(\theta^\star) = 0$, there is no ordinary Lagrange multiplier for the formulation (\ref{eq:HardINT}).%

\subsection{Duality}

The fundamental principle that underpins the prediction error of interpolators is \emph{duality}. To see this, it is instructive to first consider the running example of least-squared linear regression. Assuming that $\cov(X) = I$ with each $(x_{i}, y_i)$ independent and identically distributed, let $\theta^\star = \argmin_{\theta \in \mathbb{R}^d}\|X \theta - y\|^2$ (the least-squares estimator) in the underparameterized setting where $n > d$ and the Moore-Penrose estimator (\ref{eq:ridgeless}) in the overparameterized setting where $d > n$. The prediction error splits into the familiar bias--variance decomposition, with the variance playing a particularly prominent role in the overparameterized setting. From \cite[Theorem 1]{hastie2022surprises}, under mild conditions, the variance component of the prediction error reduces to a neat closed-form solution as $d,n\to\infty$ with $\frac{d}{n}\to c$:
\begin{equation}
\label{eq:linvariance}
\tr(\cov(\theta^\star\mid X)) \to \begin{cases}
    \displaystyle\frac{1}{1-c} \sim \displaystyle\frac{d}{n - d} & \text{ if } n > d \\
    & \\
    \displaystyle\frac{c}{1-c} \sim \displaystyle\frac{n}{d - n} & \text{ if } d > n.
\end{cases}
\end{equation}
Depending on whether $d > n$ or $d < n$, the roles of $n$ and $d$ in \eqref{eq:linvariance} appear to interchange, indicating a 
reflection between
parameter dimension and data size. We will find this is not unique to linear regression, as a similar phenomenon holds for general overparameterized models by a notion of \emph{duality}.

From the exchange of roles of $n$ and $d$ in \eqref{eq:linvariance}, we infer that in the overparameterized regime, there is an underlying \emph{underparameterized} optimization problem which can be considered to simplify calculations. More concretely, we observe that $x \cdot \theta^\star = x^\top X^+ y = y^\top (X^\top)^+ x$ and so the solution to (\ref{eq:ridgeless}) can be recast as the solution to a corresponding problem over the column space of $X$.
In particular, (\ref{eq:ridgeless}) is solved by $\theta^\star = X^\top \lambda^\star$ where
\[
\lambda^\star = \argmin_{\lambda \in \mathbb{R}^n} \|X^\top \lambda - y^\star\|^2,
\]
and $y^\star = X^+ y$. This is a special case of \emph{strong duality}, where instead of the column space of $X$ (in $d$ dimensions), the problem over $\lambda$ now occurs over the codomain of $F$ (in $n$ dimensions). This reduces the overparameterized problem back to the classical case, where it is typically easier to solve.

For more general models involving non-linear equality constraints, the primal problem is non-convex, and hence the ordinary notion of strong duality is not generally applicable. Instead, one can establish a similar zero-duality gap by considering 
an augmented Lagrangian function associated with \eqref{eq:INT}: a function $\Lambda: \Theta  \times (0,\infty) \times \mathbb{R}^{n \times m} \to \mathbb{R}$, defined as 
\begin{equation}
\label{eq:Lagrangian}
    \Lambda(\theta,\gamma, \lambda) %
    = %
    R(\theta) + \frac{1}{\gamma} L(F(\theta),y) + \sum_{i=1}^n \lambda_i \cdot (f(x_i, \theta) - y_i),
\end{equation}
where %
we have used the loss $L$ as the augmenting function \cite[Definition 11.55]{rockafellar2009variational}. Note that (\ref{eq:Lagrangian}) differs from (\ref{eq:MAP}) only by a weighted average over the residuals. By introducing dual variables $\lambda$ on the codomain, (\ref{eq:Lagrangian}) sidesteps the aforementioned issues with (\ref{eq:MAP}) and the formulation (\ref{eq:HardINT}).
Under mild conditions, a form of strong duality can be established in terms of the augmented Lagrangian for non-convex optimization \cite[Theorem 11.59]{rockafellar2009variational}. Here, we present a simplified version of this theorem adapted to our setting. 

\begin{lemma}[\textsc{Augmented Lagrangian Duality} {\cite[Theorem 11.59]{rockafellar2009variational}}]
\label{lem:Duality}
Suppose $\Theta$ is compact and $\mathcal{M}$ is non-empty. Considering the dual function $\Lambda^{\star}: (0,\infty) \times \mathbb{R}^{n \times m}  \to \mathbb{R}$ given by $\Lambda^{\star}(\gamma, \lambda) = \inf_{\theta \in \Theta} \Lambda(\theta,\gamma, \lambda)$, we have
\begin{align*}
\inf_{\theta\in \mathcal{M} \subset \mathbb{R}^{d}} \; R(\theta) = \sup_{\lambda \in \mathbb{R}^{n \times m}} \sup_{\gamma \in (0,\infty)}\Lambda^\star(\gamma,\lambda).
\end{align*}
In particular, the respective solutions $\theta^\star$, $\lambda^\star$, and $\gamma^\star$, when they exist, are related by $R(\theta^{\star}) = \Lambda(\theta^\star,\gamma^\star, \lambda^\star) = \Lambda^\star(\gamma^\star,\lambda^\star)$. 
\end{lemma}
Lemma \ref{lem:Duality} provides an alternative representation of solutions to \eqref{eq:INT} through a dual objective over the temperature $\gamma$ and the level sets of $F$. This turns an overparameterized problem \eqref{eq:INT} into an underparameterized problem. An analogous procedure of recasting the marginal likelihood in terms of the likelihood function's level sets will (shortly, in Proposition \ref{thm:DR}) form a key part of the derivation of the IIC.

\section{Bayesian Duality} 
\label{sec:MargLikeDuality}

The temperature $\gamma$ now plays a significant role, controlling the spread of the posterior, and concentrating the integral about the set of interpolators $\mathcal{M}$ as $\gamma$ becomes small. 
To measure the model quality of interpolators under the marginal likelihood, we would like to consider $\lim_{\gamma \to 0^+}\mathcal{Z}_{n,\gamma}$. For underparameterized models, this is often possible using Laplace's method \cite[\S9.1.2]{konishi2008information}. Unfortunately, the relevant assumptions inevitably fail in the overparameterized case, necessitating a dual formulation.

The derivation of our information criterion relies upon two key assumptions, the first on the integrability of the prior $\pi$ and predictor $F$, and the second on their regularity. The integrability condition requires that the prior decays sufficiently quickly over sets where $J(\theta)$ is near-singular, and the regularity assumption ensures smoothness and that the limiting object is well-defined.

\begin{assumption}[\textsc{Integrability}]
\label{ass:Base}
The interpolating manifold $\mathcal{M}$ is nonempty and the mapping $\theta \mapsto \pi(\theta) \det J(\theta)^{-1/2} \in L^1(\mathbb{R}^d)$. 
\end{assumption}

\begin{assumption}[\textsc{Regularity}]
\label{ass:LocalReg}
$F$ and $\pi$ are $\mathcal{C}^{\infty}$-smooth, $DF$ is full-rank on $\mathcal{M}$, and the marginal likelihood satisfies $\limsup_{\gamma \to 0^+} \mathcal{Z}_{n,\gamma} < +\infty$.
\end{assumption}

Assumption \ref{ass:LocalReg} rules out cases where the prior density $\pi$ is unbounded on $\mathcal{M}$. %
The smoothness assumptions are largely for convenience and can be relaxed considerably in practice, although they allow us to phrase things in terms of existing series expansions. Further, they enable the following conditions that guarantee \emph{global} regularity. We will see in Lemma~\ref{lem:Condition} in Appendix~\ref{sec:Laplace} that Assumption~\ref{ass:Base} and Condition~\ref{ass:GlobalReg} together imply Assumption~\ref{ass:LocalReg}, and can be used to dramatically simplify the proofs in certain settings.%

\begin{condition}[\textsc{Global Regularity}]
\label{ass:GlobalReg}
$F$ and $\pi$ are $\mathcal{C}^{\infty}$-smooth in an open set containing $\Theta$ and $DF$ is full-rank on $\Theta$, %
and one of the following two (disjoint) conditions holds:
\begin{enumerate}[label=(\Alph*)]
    \item \label{ass:GlobalReg_a} \emph{(\textbf{Meigniez condition}).} For any fixed $z \in F(\Theta) \subseteq \mathbb{R}^{mn}$, the preimage of $z$ \linebreak $F^{-1}(z)~=~\{\theta~\in~\Theta \,:\, F(\theta) = z\}$ is diffeomorphic to $\mathbb{R}^{d-mn}$. %
    \item \label{ass:GlobalReg_b} \emph{(\textbf{Ehresmann condition})}. For any compact set $E \subseteq \mathbb{R}^{mn}$, the preimage of $E$ \linebreak $F^{-1}(E)~=~\{\theta~\in~\Theta\,:\, F(\theta) \in E\}$ is compact.
\end{enumerate}
\end{condition}

\begin{remark*}[\textsc{Generalized Linear Models}]
It is straightforward to show that the Meigniez condition holds in the setting of linear models, and by extension, several types of generalized linear models (GLMs) as well. For the latter, note the assumption is satisfied for any model $x \mapsto \phi(A x)$ which is the composition of a bijection $\phi:\mathbb{R}^{m} \to \mathcal{X}$ and a linear map $A:\mathbb{R}^{p\times m}$ where $d = m p$ is the total number of elements in $A$. This is precisely the setting of GLMs, where $\phi$ is the inverse-link function. %
\end{remark*}

On the other hand, the Ehresmann condition does \textit{not} hold for linear models, but can be easier to verify in other cases. Much like strong duality (Lemma \ref{lem:Duality}), if $\Theta$ is compact, then the Ehresmann condition always holds. Otherwise, it is satisfied, for example, if there exists a monotone increasing function $\varphi:\mathbb{R}_+ \to \mathbb{R}_+$ such that $\varphi(x) \to \infty$ as $x \to +\infty$ and $\|F(x)\| \geq \varphi(\|x\|)$ for all $\|x\| > r$. %
We note that this is more general than conditions imposed in similar work \cite{de2021quantitative}. 

Analogous to augmented Lagrangian duality, we formulate a dual representation of the marginal likelihood by integrating over the level sets of $F$. A local approximation around an interpolator can then be obtained by taking a limit in the temperature $\gamma$, and then concentrating over the dual variables. 
Here we appeal to the coarea formula, which is stated as Theorem~\ref{thm:coarea} in Appendix~\ref{sec:gmt} for the reader's convenience. 
This result enables us to define the following representation of the marginal likelihood in terms of a \emph{dual model} over $\mathbb{R}^{mn}$.  %
We remark that such a representation is also true in the underparameterized setting, provided by the area formula (Appendix~\ref{sec:gmt}, Theorem~\ref{thm:area}), although this representation is typically unnecessary. For brevity, we let $p_{n,\gamma}(y \vert \theta, x) = \prod_{i=1}^n p_\gamma(y_i \vert \theta, x_i)$.

\begin{proposition}[\textsc{Bayesian Duality}]
\label{thm:DR}
Let Assumption \ref{ass:Base} hold and $d > mn$. The following two marginal likelihoods are \emph{equivalent}:
\begin{itemize}
    \item $\mathcal{Z}_{n,\gamma}$ corresponding to the model with likelihood $p_{n,\gamma}(y\vert x, \theta)$ and prior $\pi(\theta)$ over $\theta \in \Theta \subseteq \mathbb{R}^d$
    \item $\mathcal{Z}^\star_{n,\gamma}$ corresponding to the model with likelihood $p_{n,\gamma}^\star(y\vert z)$ and prior $\pi^\star(z)$ over $z \in \mathbb{R}^{mn}$, where
    \[
        p_{n,\gamma}^\star(y\vert z) = c_{n,\gamma}(z) e^{-\frac1\gamma L(z,y)},\qquad
        \pi^\star(z) = \int_{F^{-1}(z)} \frac{\pi(\theta)}{\det J(\theta)^{1/2}} \dd \mathcal{H}^{d-mn}(\theta),
    \]
    are probability densities on $\mathbb{R}^{mn}$, where $\mathcal{H}^\alpha$ is $\alpha$-dimensional Hausdorff measure.
\end{itemize}
In other words,
\begin{equation}
\tag{DR}
\label{eq:DR}
\int_{\Theta \subseteq \mathbb{R}^d} p_{n,\gamma}(y\vert x,\theta) \pi(\theta) \dd \theta = \int_{\mathbb{R}^{mn}} p_{n,\gamma}^\star(y \vert z) \pi^\star(z) \dd z.
\end{equation}
\end{proposition}

\begin{proof}[Proof of Proposition \ref{thm:DR}]
The expression (\ref{eq:DR}) follows by direct application of the coarea formula: note that
\begin{align*}
\mathcal{Z}_{n,\gamma} = \int_{\Theta} p_{n,\gamma}(y \vert \theta,x) \pi(\theta) \dd \theta
&= \int_{\Theta} c_{n,\gamma}(F(\theta)) e^{-\frac{1}{\gamma} L(F(\theta),y)} \pi(\theta) \dd \theta \\
&= \int_{\mathbb{R}^{mn}} c_{n,\gamma}(z) e^{-\frac{1}{\gamma} L(z,y)} \pi^\star (z) \dd z ,
\end{align*}
where
\[
\pi^\star(z) = \int_{F^{-1}(z)} \frac{\pi(\theta)}{\det J(\theta)^{1/2}} \dd \mathcal{H}^{d - mn}(\theta).
\]
Denoting $p_{n,\gamma}^\star(y \vert z) = c_{n,\gamma}(z) e^{-\frac1\gamma L(z,y)}$, $p_{n,\gamma}^\star$ is a density in $y$. Furthermore,
$
\mathcal{Z}_{n,\gamma} = \mathcal{Z}^\star_{n,\gamma} $ where $\mathcal{Z}^\star_{n,\gamma}= \int_{\mathbb{R}^{mn}} p_{n,\gamma}^\star (y \vert z) \pi^\star (z) \dd z.
$
In order to see that $\mathcal{Z}^\star_{n,\gamma}$ is itself a marginal likelihood, we need to show that $\pi^\star$ is a probability density. However, this is immediate, since a further application of the coarea formula tells us that $1 = \int_{\Theta} \pi(\theta) \dd \theta = \int_{\mathbb{R}^{mn}} \pi^\star(z) \dd z$. %
\end{proof}

\begin{example*}[\textsc{Overparameterized Linear Regression}]
We return to the setting of least-squares linear regression with $m=1$. Note that $F^{-1}(z) = \{X^+ z + w\,:\, w \in \ker(X)\}$. %
The Meigniez condition is satisfied if $X$ is full rank, whereby under the rank--nullity theorem, $\ker(X)$ is a $(d-n)$-dimensional vector space. %
Indeed, there exists a semi-orthogonal matrix $Q \in \mathbb{R}^{d\times (d-n)}$ such that $\ker(X) = \text{Range}(Q)$,
and the dual prior is given by the $(d-n)$-dimensional Radon transform of $\pi$:
\[
\pi^\star(z) = \frac{1}{\det(XX^\top)^{1/2}} \int_{\mathbb{R}^{d-n}} \pi(X^+z + Q w) \dd w.
\]
Duality follows from the Fourier transform $\mathcal{F}$ on $L^1$, as $(\mathcal{F}\pi^{\ast})(z)=(\mathcal{F}\pi)(X^{\top}z)$ for $z \in \mathbb{R}^n$. If $\pi$ is the density of a zero-mean normal distribution with covariance $\tau I$, then $\pi^\star$ is the density of a zero-mean normal distribution with covariance $\tau X X^\top$, or equivalently,
\begin{equation}
\label{eq:DualPriorLinear}
\pi^\star(z) = \frac{1}{(2\pi\tau)^{n/2}\det(XX^{\top})^{1/2}} \exp\left(-\frac1{2\tau}\|X^+ z\|^2\right).
\end{equation}
\end{example*}

Proposition~\ref{thm:DR} asserts that in the overparameterized regime, there exists a dual model over the level sets of $F$ with the same marginal likelihood. The mechanics of the proof of Proposition \ref{thm:DR} are not new; for example, the area and coarea formulae have been used in Monte Carlo methods to study algorithms moving between level sets \cite{CG16}, the existence of a limiting posterior distribution in the low temperature limit $\gamma \to 0^+$ \cite{AGT23,de2021quantitative,diaconis2013sampling}, studying stochastic optimization methods for overparameterized models \cite{de2021quantitative}, and other computations of free energy \cite[\S3.3.2]{stoltz2010free}.
To our knowledge however, the characterization of (\ref{eq:DR}) as a form of duality and the conditions for regularity of $\pi^\star$, which enable the following proposition, are both novel.

\begin{proposition}[\textsc{Smoothness of the Dual Prior}]
\label{thm:regularity}
Suppose that $d>mn$ under Assumption \ref{ass:Base}, and either the Meigniez or the Ehresmann condition holds.
Then $\pi^\star$ is $\mathcal{C}^{\infty}$-smooth on $F(\Theta)$.
\end{proposition}

\begin{proof}
Let $z_0 \in F(\mathbb{R}^d)$ be arbitrary and let $\Omega = F^{-1}(z_0)$. Since Condition~\ref{ass:GlobalReg} asserts that $DF$ is full-rank on $\mathbb{R}^d$,  the Submersion Theorem \cite[Corollary 5.13]{lee2012smooth} implies that $F^{-1}(z)$ is a $(d-mn)$-submanifold in $\mathbb{R}^d$ for any $z$ in the image of $F$. Therefore,
\[
\int_{F^{-1}(z)} g(\theta) \dd \mathcal{H}^{d-mn}(\theta) = \int_{F^{-1}(z)} g(\theta) \dd V, \quad\text{where } g(\theta) = \frac{\pi(\theta)}{\det J(\theta)^{1/2}},
\]
and $\dd V$ is the associated volume form. Both the Ehresmann condition (Theorem \ref{thm:Ehresmann}) and the Meigniez condition (Theorem \ref{thm:Meigniez}) imply that there is a neighbourhood $N_0 \subseteq \mathbb{R}^{mn}$ of $z_0$ and a $\mathcal{C}^{\infty}$-smooth map $\varphi\,:\,N_0 \times \mathbb{R}^d$ such that $\varphi(z,\Omega) = F^{-1}(z)$ for all $z \in N_0$. Therefore,
\[
\pi^\star(z) = \int_{F^{-1}(z)} g(\theta) \dd V = \int_{\varphi(z,\Omega)} g(\theta) \dd V.
\]
From the Leibniz integral rule, Theorem \ref{thm:Leibniz}, $\pi^\star$ is $\mathcal{C}^{\infty}$-smooth in $N_0$. Since $z_0$ was arbitrary, $\pi^\star \in \mathcal{C}^{\infty}(F(\mathbb{R}^d))$.
\end{proof}

\section{The Interpolating Information Criterion}\label{sed:IIC}

Using Propositions~\ref{thm:DR} and \ref{thm:regularity}, it is possible to study properties of overparameterized systems using classical techniques. 
For our purposes, one of the most significant use cases of Propositions \ref{thm:DR} and \ref{thm:regularity} is that they readily enable expansions of the marginal likelihood in the temperature $\gamma$ via Laplace's method. Since $p_{n,\gamma}^\star$ is often locally log-concave, the dual formulation (\ref{eq:DR}) allows us to easily consider the marginal likelihood in the \textit{cold posterior} limit as $\gamma$ approaches $0$ from above --- a distribution 
whose support is $\mathcal{M}$ \cite{de2021quantitative}. 

This limit has particular significance in machine learning settings for models obtained as a result of stochastic optimization, where $\gamma$ represents the annealed temperature of the optimization as it is reduced to zero, concentrating solutions onto the set of optima \cite{mandt2017stochastic,kenyon,robbins1951stochastic}. %
In this regime, our dual representation \eqref{eq:DR} allows us to follow similar steps to the derivation of other Bayesian information criteria \cite{konishi2008information} to construct a new information criterion. However, since $\pi^\star$ is also generally intractable to compute explicitly, we need to invoke another approximation to obtain a point estimate. To do so, we will consider a family of concentrating priors $\{\pi_\tau\}_{\tau \in (0,1]}$ satisfying $\pi_\tau(\theta) \propto \pi(\theta)^{1/\tau}$.
This lets us consider corresponding families of concentrating marginal likelihoods and their approximations on compact sets, given by
\begin{align*}
\mathcal{Z}_{n,\gamma,\tau} = \int_{\Theta} p_{n,\gamma}(y \vert x,\theta) \pi_\tau(\theta) \dd \theta
\qquad\text{and}\qquad
\mathcal{Z}_{n,\gamma,\tau}^K = \int_{\mathbb{R}^{mn}} p_{n,\gamma}^\star(y \vert z) \pi^\star_{K,\tau}(z) \dd z,
\end{align*}
where we naturally define
\[
\pi^\star_{K,\tau}(z) = \int_{F^{-1}(z) \cap K} \frac{\pi_\tau(\theta)}{\det J(\theta)^{1/2}} \dd \mathcal{H}^{d-mn}(\theta).
\]
Naturally, we also let $\mathcal{F}_{n,\gamma,\tau} = -\log \mathcal{Z}_{n,\gamma,\tau}$ and $\mathcal{F}_{n,\gamma,\tau}^K  = -\log \mathcal{Z}_{n,\gamma,\tau}^K$. Our first technical lemma demonstrates regularity of the approximation $\pi_{K,\tau}^{\star}$ around $z = y$.
\begin{lemma}
\label{lem:CompactApprox1}
Let $K \subseteq \Theta$ be a compact set 
 such that $K \cap \mathcal{M} \neq \emptyset$. Under Assumptions \ref{ass:Base} and \ref{ass:LocalReg}, $\pi_{K,\tau}^{\star}$ is $\mathcal{C}^{\infty}$-smooth in a neighbourhood of $z = y$.
\end{lemma}
\begin{proof}
Let $F_K\,:\,K \to F(K)$ denote the restriction of $F$ to the set $K$. By hypothesis, $F_K$ is a submersion in a neighbourhood $N$ of $\mathcal{M} \cap K$. For $E \subset F(K)$ compact, we have that $F_K^{-1}(E) = F^{-1}(E) \cap K$, which is bounded and closed by continuity of $F$ (Assumption \ref{ass:Base}), and so by the Heine-Borel Theorem, $F_K^{-1}(E)$ is compact. The Ehresmann condition is satisfied for $F_K$, and so there exists a neighbourhood $N_0$ of $y$ and a smooth map $\varphi:N_0 \times \mathcal{M}$ such that $\varphi(z,\mathcal{M}) = F^{-1}(z)$ for $z \in N_0$. %
As in the proof of Proposition~\ref{thm:regularity}, the Leibniz integral rule implies that $\pi_{K,\tau}^\star$ is $\mathcal{C}^{\infty}$-smooth in $N_0 \cap F(N)$.
\end{proof}

There are now two temperatures $\gamma$ and $\tau$---reducing them at differing rates is known to produce different approximations \cite{fulks1951generalization}. If $\gamma$ and $\tau$ are reduced at similar rates, the marginal likelihood will concentrate around (\ref{eq:MAP}). However, in light of Proposition \ref{thm:DR}, a simpler approach first invokes the cold posterior limit $\gamma \to 0^+$, after which $\tau$ is reduced. This instead concentrates the marginal likelihood around (\ref{eq:INT}), suggesting a criterion that is well-equipped to compare model performance among interpolators. Consequently, we refer to this as the \emph{interpolating regime}. %
To proceed we will need to impose some mild assumptions on the second order behavior of $\ell$, and the second order and asymptotic behavior of the base prior $\pi$. This will ensure that the Laplace approximation and the curvature terms appearing in the IIC are well-defined, with the prior concentrating fast enough to take approximations on compacta. 

To perform a Laplace approximation over the submanifold $\mathcal{M}$ requires the manifold Hessian. Let $\Pi(\theta) = I - DF(\theta)^\top J(\theta)^{-1} DF(\theta)$ denote the projection matrix mapping $\Theta$ into the tangent space of $\mathcal{M}$ at $\theta$, and write its compact singular value decomposition as $U_\theta U_\theta^\top$, so $U_\theta$ maps $\mathbb{R}^{d-mn}$ into the tangent space of $\mathcal{M}$ at $\theta$. For $\theta \in \mathcal{M}$ and twice-differentiable $R:\mathbb{R}^d \to \mathbb{R}$, we define the manifold Hessian $\nabla_{\mathcal{M}}^2 R(\theta) \in \mathbb{R}^{(d-mn)\times (d-mn)}$ through its action on vectors $u \in \mathbb{R}^{d-mn}$:
\begin{equation}
\label{eq:ManifoldHessian}
\nabla^2_{\mathcal{M}} R(\theta) u = U_\theta^\top (\nabla^2 R(\theta) U_\theta u - \nabla_{U_\theta u} \Pi(\theta) \nabla R(\theta))\;:\; \mathcal{M} \times \mathbb{R}^{d-mn} \to \mathbb{R}^{d-mn}.
\end{equation}
Note that $\nabla_{\mathcal{M}}^2 R(\theta) = I$ in the setting of linear regression with $R(\theta) = \frac12 \|\theta\|^2$. 
This allows us to derive the following Laplace approximation on constrained submanifolds.

\begin{proposition}[\textsc{Laplace Approximation on Constrained Submanifolds}]
\label{prop:LaplaceManifold}
Suppose that $\mathcal{M} = F^{-1}(y)$ where $F:\mathbb{R}^d \to \mathbb{R}^{mn}$. Let $\eta,Q:\mathbb{R}^d \to [0,\infty)$ be smooth, and assume $\eta$ attains a unique global minimum on $\mathcal{M}$ at $\theta^\star$ and $\nabla_{\mathcal{M}}^2 \eta$ is non-singular. Then %
\begin{equation}
\label{eq:LaplaceManifold}
\int_{\mathcal{M}} e^{-\frac1\tau \eta(\theta)} Q(\theta) \dd \mathcal{H}^{d-mn}(\theta) = (2\pi\tau)^{\frac{d-mn}{2}} e^{-\frac1\tau \eta(\theta^\star)} Q(\theta^\star) \det(\nabla^2_{\mathcal{M}} \eta(\theta^\star))^{-1/2}[1 + \mathcal{O}(\tau)].
\end{equation}
\end{proposition}
\begin{proof}
Let $\mathcal{U}_0$ be an open set on $\mathcal{M}$ such that $\theta^\star \in \mathcal{U}_0$ and the exponential map $\mathrm{Exp}_{\theta^\star}$ is a diffeomorphism from some neighbourhood $\mathcal{V}_0$ of $T_{\theta^\star}\mathcal{M}$ into $\mathcal{U}_0$. Similarly, let $\{\mathcal{U}_\alpha\}$ be a collection of open sets on $\mathcal{M}\setminus \mathcal{U}_0$ such that for each $\mathcal{U}_\alpha$, there is a point $\theta_\alpha$ such that $\mathrm{Exp}_{\theta_\alpha}$ is a diffeomorphism from some neighbourhood $\mathcal{V}_\alpha$ of $T_{\theta_\alpha}\mathcal{M}$ into $\mathcal{U}_\alpha$. The existence of these sets follows from local existence of the exponential map \cite[Proposition 8.2]{kobayashi1963foundations}. By \cite[Theorem 2.23]{lee2012smooth}, there exists a smooth partition of unity $\{\psi_0\}\cup\{\psi_\alpha\}$ of $\mathcal{M}$ over the sets $\{\mathcal{U}_0\}\cup\{\mathcal{U}_\alpha\}$, which can be modified to ensure that $\psi_0(\theta) = 1$ on some open neighbourhood $N_{\theta^\star}\subset\mathcal{U}_0$ containing $\theta^\star$. For each $\alpha$, let $\phi_\alpha(w) = \mathrm{Exp}_{\theta_\alpha}(U_\alpha w)$, where $U_\alpha$ is the semi-orthogonal matrix mapping $\mathbb{R}^{d-n}$ into $T_{\theta_\alpha}\mathcal{M}$. Define $\phi_0(w)$ similarly, replacing $U_\alpha$ with $U_0$ and $T_{\theta_\alpha}\mathcal{M}$ with $T_{\theta^\star}\mathcal{M}$. Then,
\[
\int_{\mathcal{M}}e^{-\frac1\tau \eta(\theta)} Q(\theta) \dd V = I_1 + \sum_\alpha I_{2,\alpha},
\]
where
\[
\begin{aligned}
    I_1 &= \int_{U_0^\top \mathcal{V}_0} e^{-\frac1\tau \eta(\phi_0(w))} Q(\phi_0(w)) \psi_0(\phi_0(w)) \sqrt{\det G_0(w)} \dd w\\
    I_{2,\alpha} &= \int_{U_\alpha^\top \mathcal{V}_\alpha} e^{-\frac1\tau \eta(\phi_\alpha(w))} Q(\phi_\alpha(w)) \psi_\alpha(\phi_\alpha(w)) \sqrt{\det G_\alpha(w)} \dd w, 
\end{aligned}
\]
with $G_\alpha(w) = D\phi_\alpha(w)^\top D\phi_\alpha(w)$ and similarly for $G_0(w)$. Applying Laplace's method 
\cite[Theorem 15.2.5]{simon2015advanced} to $I_1$, since $\eta \circ \phi_\alpha$ attains a unique minimum at zero, %
\[
I_1 = (2\pi\tau)^{\frac{d-mn}{2}} e^{-\frac1\tau \eta(\theta^\star)} Q(\theta^\star) \det(\nabla^2(\eta \circ \phi)(\theta^\star))^{-1/2}[1 + \mathcal{O}(\tau)],
\]
which is equal to the right-hand side of (\ref{eq:LaplaceManifold}) by \cite[Proposition 5.45]{boumal2023introduction}.
Now consider each $I_{2,\alpha}$, and observe that for any $k > 1$,
\[
\tau^{-k} I_{2,\alpha} e^{\frac1\tau \eta(\theta^\star)} = \int_{U_\alpha^\top \mathcal{V}_\alpha} \tau^{-k} e^{-\frac1\tau [\eta(\phi_\alpha(w)) - \eta(\theta^\star)]} Q(\phi_\alpha(w)) \psi_\alpha(\phi_\alpha(w)) \sqrt{\det G_\alpha(w)} \dd w.
\]
Since $e^{-[\eta(\phi_\alpha(w)) - \eta(\theta^\star)]} < 1-\delta$ for some $\delta > 0$ on $U_\alpha^\top\mathcal{V}_\alpha$, for any $0 < \tau < 1$ there holds
\[
\tau^{-k} I_{2,\alpha} e^{\frac1\tau \eta(\theta^\star)} < \tau^{-k} (1-\delta)^{\frac{1}{\tau}-1} \int_{U_\alpha^\top \mathcal{V}_\alpha} e^{-[\eta(\phi_\alpha(w)) - \eta(\theta^\star)]} Q(\phi_\alpha(w)) \psi_\alpha(\phi_\alpha(w)) \sqrt{\det G_\alpha(w)} \dd w,
\]
which converges to zero as $\tau \to 0^+$. Hence, $I_{2,\alpha} = o(\tau^k e^{-\frac1\tau \eta(\theta^\star)})$ as $\tau \to 0^+$ for any $k > 0$ and each $\alpha$, which implies the result.
\end{proof}

\begin{assumption}
\label{ass:UniquePrior}
There are \textbf{unique} parameters $\theta_0,\theta^\star \in \Theta \subseteq \mathbb{R}^d$ satisfying 
\[
\theta_0 = \argmax_{\theta \in \Theta} \pi(\theta),\qquad 
\theta^\star = \argmax_{\theta \in \mathcal{M}} \pi(\theta).
\]
In particular, $\theta^\star$ is the unique interpolator solving (\ref{eq:INT}) with $R(\theta) = -\log \pi(\theta)$. Furthermore, 
\begin{enumerate}[label=(\alph*)]
    \item the Hessian $\nabla_z^2 \ell(z,y)$ is non-singular at $z = y$;
    \item the manifold Hessian $\nabla_{\mathcal{M}}^2 R(\theta^\star)$ is non-singular; and 
    \item the Hessian $\nabla^2 R(\theta_0)$ is non-singular.
\end{enumerate}

\end{assumption}

The following technical lemma establishes some decay properties of $\pi_\tau$ outside a compact set and quantifies the approximation to the corresponding marginal likelihood on these sets.
\begin{lemma}
\label{lem:PriorProps}
Under Assumptions \ref{ass:LocalReg} and \ref{ass:UniquePrior}, there exists a compact set $K \subset \mathbb{R}^d$ with $\mathcal{H}^{d-mn}(K \cap \mathcal{M}) > 0$ %
and $\theta_0 \in K$, and constants $0 < \varrho < 1$ and $C > 0$, such that $|\mathcal{F}_{n,\gamma,\tau} - \mathcal{F}_{n,\gamma,\tau}^K| \leq C \varrho^{1/\tau}$ for sufficiently small $\gamma,\tau > 0$.%
\end{lemma}

\begin{proof}
It will suffice to find a bounded open set $K$ with the required results. To establish the conditions on $K$ we let $r_0 = 1 + \|\theta_0\|$, and $r_1$ be sufficiently large so that $0 \in F(B_{r_1})$, where $B_r$ denotes the open ball of radius $r$ about the origin. Choose $K = B_{\max\{r_0,r_1\}}$. Under the stated assumptions, the conditions of the Compactification Lemma of \cite[Lemma 38]{breitung2006asymptotic} are satisfied for $K$ with $h(\theta) = p_{n,\gamma}(y\vert\theta,x)$ and $f(\theta) = R(\theta)$, and so from \cite[eqn. 5.14]{breitung2006asymptotic},
\[
\frac{\int_{\Theta \setminus K} p_{n,\gamma}(y\vert\theta,x) \pi_\tau(\theta) \dd \theta}{\int_{\Theta \cap K} p_{n,\gamma}(y\vert\theta,x) \pi_\tau(\theta) \dd \theta} \leq e^{-\frac{\epsilon}{2\tau}} \frac{\int_{\Theta \setminus K} p_{n,\gamma}(y\vert\theta,x)\pi(\theta) \dd \theta}{\int_{\Theta \cap K} p_{n,\gamma}(y\vert \theta,x)\pi(\theta) \dd \theta},
\]
where $\epsilon = R(\theta_0) - \sup_{\theta \in \Theta\setminus K}R(\theta) > 0$. Letting $\varrho = e^{-\frac{\epsilon}{2}} \in (0,1)$, in terms of $\mathcal{Z}_{n,\gamma,\tau}$ and $\mathcal{Z}_{n,\gamma,\tau}^K$,
\[
\frac{\mathcal{Z}_{n,\gamma,\tau}}{\mathcal{Z}_{n,\gamma,\tau}^K} - 1 \leq \varrho^{1/\tau} \left(\frac{\mathcal{Z}_{n,\gamma,1}}{\mathcal{Z}_{n,\gamma,1}^K}-1\right).
\]
Since $\mathcal{H}^{d-mn}(K\cap\mathcal{M}) > 0$, $\liminf_{\gamma \to 0^+} \mathcal{Z}_{n,\gamma,1}^K > 0$ and so the right-hand side is uniformly bounded for sufficiently small $\gamma > 0$ by $C \varrho^{1/\tau}$. The result follows from the fact that $\log(1 + x) \leq x$ for $x > 0$.

\end{proof}

We are now in a position to state the main result of this work. The following theorem gives us a first order asymptotic approximation of the Bayes free energy in the interpolating regime.

\begin{theorem}
\label{thm:IIC}
Let $\mathcal{F}_{n,\gamma,\tau} = -\log \mathcal{Z}_{n,\gamma,\tau}$. In the regime where $\gamma \to 0^+$, and then $\tau \to 0^+$, and under Assumptions~\ref{ass:Base}, \ref{ass:LocalReg} and \ref{ass:UniquePrior}, we have
    \[
    \mathcal{F}_{n,\gamma,\tau} = \frac{1}{\tau} [R(\theta^\star) - R(\theta_0)] + \frac12 \log \det J(\theta^\star) + \frac{mn}{2}\log(\tau\pi) + \frac12 \log \mathcal{K}_{\mathcal{M}}^\pi(\theta^\star,\theta_0) + \mathcal{O}(\gamma) + \mathcal{O}(\tau),
    \]
    where the relative curvature factor $\mathcal{K}_{\mathcal{M}}^\pi$ is
\begin{align}\label{eq:curve}
    \mathcal{K}_{\mathcal{M}}^\pi(\theta_1,\theta_2) = \frac{\det(\nabla^2_\mathcal{M}R(\theta_1))}{\det(\nabla^2 R(\theta_2))}. 
\end{align}

\end{theorem}

\begin{proof}[Proof of Theorem \ref{thm:IIC}]
Under Assumption \ref{ass:Base}, from Proposition \ref{thm:DR}, 
$\mathcal{Z}_{n,\gamma,\tau} = \int_{\mathbb{R}^{mn}} p_\gamma^\star(y \vert z) \pi^\star_{K,\tau}(z) \dd z + (\mathcal{Z}_{n,\gamma,\tau} - \mathcal{Z}^K_{n,\gamma,\tau})$. 
Applying the Laplace approximation (Lemma \ref{lem:Laplace}), since the Hessian of $\ell$ is nonsingular, for any function $g$ that is smooth at $y$, 
\begin{align*}
\int_{\mathbb{R}^{n\times m}} p_{n,\gamma}^\star(y \vert z) g(z) \dd z 
&= \int_{\mathbb{R}^{n\times m}}\frac{e^{-\frac{1}{\gamma}L(y,z)}g(z)}{\int_{\mathbb{R}^{n\times m}}e^{-\frac{1}{\gamma}L(y,z)}\dd y} \dd z \\ 
&= \frac{(2\pi\gamma)^{n/2}g(y)\prod_{i=1}^{n}\det(\nabla^{2}\ell(y_{i},y_{i}))^{-1/2}[1+\mathcal{O}(\gamma)]}{(2\pi\gamma)^{n/2}\prod_{i=1}^{n}\det(\nabla^{2}\ell(y_{i},y_{i}))^{-1/2}[1+\mathcal{O}(\gamma)]},
\end{align*}
and so
$\int_{\mathbb{R}^{n\times m}} p_{n,\gamma}^\star(y \vert z) g(z) \dd z =  g(y)[1 + \mathcal{O}(\gamma)]$ as $\gamma \to 0^+$. Since $\pi_{K,\tau}^\star$ is smooth in a neighborhood of $y$, this and Lemma \ref{lem:PriorProps} imply that for any $k$, $\mathcal{Z}_{n,\gamma,\tau} = \pi^\star_{K,\tau}(y)[1 + \mathcal{O}(\gamma)] + o(\tau^k)$. It now only remains to estimate
\[
\pi_{K,\tau}^{\star}(y) = \int_{\mathcal{M}} \frac{\pi_\tau(\theta)}{\det J(\theta)^{1/2}} \dd \mathcal{H}^{d-mn}(\theta) = \frac{\int_{\mathcal{M}} e^{-\frac1\tau R(\theta)} \det J(\theta)^{-1/2} \dd \mathcal{H}^{d-mn}(\theta)}{\int_{\Theta} e^{-\frac1\tau R(\theta)} \dd \theta}.
\]
Applying Proposition \ref{prop:LaplaceManifold} to the numerator integral and Laplace's method to the denominator,
\[
\begin{aligned}
\int_{\mathcal{M}} e^{-\frac1\tau R(\theta)} \det J(\theta)^{-1/2} \dd \mathcal{H}^{d-mn}(\theta) &= (2\pi\tau)^{\frac{d-mn}{2}} e^{-\frac1\tau R(\theta^\star)} \det (\nabla_{\mathcal{M}}^2 R(\theta^\star))^{-1/2} \det J(\theta^\star)^{-1/2} [1 + \mathcal{O}(\tau)] \\
\int_{\Theta} e^{-\frac1\tau R(\theta)} \dd \theta &= (2\pi\tau)^{\frac{d}{2}} e^{-\frac1\tau R(\theta_0)} \det(\nabla^2 R(\theta_0))^{-1/2}[1 + \mathcal{O}(\tau)].
\end{aligned}
\]
The result now follows from $\mathcal{F}_{n,\gamma,\tau} = -\log \pi_{K,\tau}^\star(y) + \mathcal{O}(\gamma) + o(\tau^k)$ where
\[
-\log \pi^\star_{K,\tau}(y) = \frac1\tau[R(\theta^\star) - R(\theta_0)] + \frac{mn}{2} \log \tau + \frac12 \log \det J(\theta^\star) + \frac12 \log \mathcal{K}_{\mathcal{M}}^\pi(\theta^\star,\theta_0) + \mathcal{O}(\tau).
\]
\end{proof}
\begin{remark}
\label{rem:lossequal}
It is important to observe that the asymptotics for $\mathcal{F}_{n,\gamma,\tau}$ do not depend on $\ell$ outside of the interpolating set $\mathcal{M}$. Consequently, any two loss functions that interpolate on the same set will have equivalent leading order asymptotics for the free energy. 
\end{remark}

This leads to the definition of the interpolating information criterion. 
\begin{definition}[\textsc{IIC}]
\label{def:IIC}
Let $\bar{\mathcal{F}}_{n,\tau}$ denote the approximation to the Bayes free energy $\mathcal{F}_{n,\gamma,\tau}$ obtained from Theorem \ref{thm:IIC} by discarding higher-order terms. The interpolating information criterion is given~by
\begin{equation}
\label{eq:IICDefn}
\text{IIC} \simeq \frac{2}{N} \inf_{\tau} \bar{\mathcal{F}}_{n,\tau},
\end{equation}
where $\simeq$ denotes equivalence after all constant terms in $N$ and $d$ are removed. 
\end{definition}
Equation (\ref{eq:IICDefn}) mimics the construction of the Bayesian information criterion (see \cite[\S9.1]{konishi2008information} for example). 
As is standard practice with information criteria \cite[pg. 226]{konishi2008information}, we choose the auxiliary parameter $\tau$ which minimizes our approximation to the Bayes free energy. This procedure is tantamount to empirical Bayes estimation \cite{casella1985introduction}, and is essential to ensure the qualitative behavior of the free energy follows that of the posterior predictive losses in the overparameterized regime \cite{hodgkinson2022monotonicity}. 
A routine calculation reveals $\tau^\star = \frac{2}{N} \log \frac{\pi(\theta_0)}{\pi(\theta^\star)}$ as the optimal value.
The division of the free energy by a factor of $N$ 
is done in view of the fact that 
under the unique scoring rule that is continuous, monotone, and coherent, the $p$-average of the leave-$p$-out cross-validation errors is given by $N^{-1} \mathcal{F}_{n,\gamma,\tau}$ \cite{fong2020marginal}. Computing the IIC under (\ref{eq:IICDefn}) and removing constant terms yields \eqref{eq:IIC}. 

\begin{example*}[\textsc{IIC for Linear Regression}]
Returning again to the setting of least-squares regression with $\theta^\star$ given by the Moore-Penrose pseudoinverse, $\Theta = \mathbb{R}^d$, $R(\theta) = \frac12 \|\theta\|^2$,  $DF(\theta^\star) = X$, and $\mathcal{K}_{\mathcal{M}}^{\pi} = 1$. The IIC is equal to $\inf_{\tau}\lim_{\gamma \to 0^+} \mathcal{F}_{n,\gamma,\tau}$ by \cite[eq. 4]{hodgkinson2022monotonicity} and simplifies to
\begin{equation}
\label{eq:IICLinearReg}
\text{IIC} = 2 \log \|X^+ y\| + \frac{1}{n} \log \det(XX^\top) - \log n,
\end{equation}
and is to be compared to the Bayesian information criterion \cite{schwarz1978estimating} (normalized by $n$) when $n \gg d$:
\begin{equation}
\label{eq:BICLinearReg}
\text{BIC} = 2\log\|(I-XX^{+})y\|+\left(\frac{d}{n}-1\right)\log n.
\end{equation}
\end{example*}

\section{Discussion}\label{sec:Wider}

While Theorem~\ref{thm:IIC} and the associated IIC \eqref{eq:IIC} constitute our main results, several aspects of our approach may be of independent interest, in particular the phenomenon of Bayesian duality in Proposition \ref{thm:DR}.
We conclude by describing how the IIC and Bayesian duality fit in the context of related work.

\subsection{Interpreting the IIC}
\label{sec:Interpretation}

The IIC itself consists of four terms, each penalizing different characteristics of the estimator $\theta^\star$. %
The $\log N$ correction term is clear, so we discuss the remaining in order of their appearance in (\ref{eq:IIC}). For illustration, each term is represented in Figure \ref{fig:intuition_terms}, considering $R(\theta)$ as quadratic.
\begin{enumerate}[label=(\Roman*)]
\item \textbf{Regularization.} %
Since the BIC is constructed in the large-data limit, the Bernstein-von Mises Theorem \cite{le1953some} suggests that terms involving the prior are asymptotically irrelevant. However, the IIC is non-asymptotic in $d$ and $n$. Furthermore, since the likelihood cannot distinguish between points on $\mathcal{M}$, the influence of the prior becomes magnified in the interpolating regime. Therefore, the first term of the IIC plays a similar role to the log-likelihood in the BIC, %
but instead penalizes \emph{prior} mispecification, judging performance by the degree to which the point estimate obtained after fitting to data agrees with the initial data-independent measure $\pi$ (see Figure \ref{fig:intuition_terms} left). 

\item \textbf{Sharpness.} 
One of the most popular theories for overparameterized models emphasizes the importance of \emph{sharpness} \cite{hochreiter1997flat,keskar2016large}: estimators located in regions with flatter log-likelihood typically exhibit reduced test error (see Figure \ref{fig:intuition_terms} center). While typically quantified using the Hessian of the log-likelihood \cite{kaur2023maximum,yao2020pyhessian,yao2018hessian,zhang2018energy}, the Hessian is singular when $d > n$, and is therefore inadequate for measuring local volume \cite{bergamin2023riemannian,wei2022deep,zhang2018energy}. %
The second term of the IIC provides a valid quantifier of sharpness in the interpolating regime involving 
$J(\theta^\star)$, which often satisfies the full-rank assumption in practice. 
In deep learning, $J(\theta^\star)$ is the Gram matrix of the so-called \emph{neural tangent kernel} (NTK) at $\theta^\star$ \cite{jacot2018neural,novak2022fast}. NTKs can be interpreted as linearized approximations of wide neural networks, useful for studying many properties of neural networks both during \cite{fort2020deep,velikanov2021explicit}, and after training \cite{d2020triple,huang2020deep}. As observed in \cite{de2021quantitative}, the determinant is directly connected to prior notions of sharpness, as $J(\theta^\star)$ has the same nonzero eigenvalues, %
up to a scalar factor, as the Hessian of the loss at $\theta^\star$. Consequently, the interpretation of $J(\theta^\star)$ is the same as the underparameterized setting: the smaller its determinant, the flatter the log-likelihood in the neighborhood of $\theta^\star$. Since the determinant depends only on the eigenvalues of the Jacobian, the sharpness term completely encodes spectral information in the model \cite{liao2021hessian,martin2021implicit} relevant for predicting performance \cite{martin2021predicting}. 

\item \textbf{Relative Curvature.} While the regularization term encapsulates global properties of the prior, the relative curvature term captures local incompatibility between the prior and the log-likelihood, by comparing the curvature of $R$ over the ambient space $\Theta$ to its curvature over $\mathcal{M}$
(see Figure \ref{fig:intuition_terms} right). The manifold Hessian $\nabla_{\mathcal{M}}^2 R$ is \emph{intrinsic} to the zero-loss set $\mathcal{M}$, and unlike the sharpness term, does not depend on any properties of $F$ other than its specification of $\mathcal{M}$. From an extrinsic viewpoint, the manifold Hessian depends on the second fundamental form \cite[eq. 5.42]{boumal2023introduction}, and so relates to other notions of curvature via the Gauss--Codazzi equations \cite[Proposition 4.1]{kobayashi1969foundations}. 
The influence of geometric properties of the likelihood on model performance is well-studied \cite{amari2016information}, including other types of Riemannian Laplace approximations \cite{bergamin2023riemannian}. The IIC expands on these connections as they arise from the zero-temperature limit.
\end{enumerate}

\subsection{Double Descent and the Occam Factor}
\label{sec:DD}

Bayesian methods inherently embody Occam's razor (see e.g., \cite{mackay1992bayesian}) as the marginal likelihood depends strongly on the volume around the maximum \emph{a posteriori} estimator.
As per the bias-variance tradeoff, the estimator prediction error first decreases with reduced bias as the model class becomes more complex, and then past some optimal point increases due to variance.  %
For $n \gg d$, the tradeoff is reflected in the $\mathcal{O}(d\log n)$ penalty term in the BIC. However, even for least-squares linear regression, the peak in prediction error observed around $d=n$ may be transient---appropriately chosen estimators for well-specified but overparameterized models can in fact improve a model's predictive capacity as $d$ grows indefinitely. This phenomenon is commonly referred to as either \emph{double descent} \cite{AS17_TR,belkin2019reconciling,derezinski2020exact,liao2020random} or \emph{benign overfitting} \cite{bartlett2020benign,tsigler2020benign}. %
Results of this flavor are typically proven in the setting where $d \sim c n$ as $d,n\to\infty$ \cite{hastie2022surprises}. 
For finite-dimensional Gaussian processes with fixed $n$, the marginal likelihood itself turns out to be monotone in $d$ when the prior hyperparameter (represented by $\tau$ in Section \ref{sed:IIC}) is optimized via an empirical Bayes procedure \cite{hodgkinson2022monotonicity}. Thus, a litmus test for the prospective utility of any information criterion in this regime is that when applied to overparameterized linear regression (including random feature models), it should reflect these behaviors. Figure \ref{fig:DDRFF}---whose details are outlined in the following section---demonstrates that the IIC passes this test. 

However, whilst the example provided is considered for the sake of concreteness, IIC applies far beyond the case of linear/kernel regression. 
Indeed, Bayesian duality (Proposition \ref{thm:DR}) helps explain the widely-observed occurrence of double descent and benign overfitting phenomena. 
To see this, assume the likelihood is of the usual form $p(y\vert X,\theta) = c_n e^{-n \mathcal{L}(\theta)}$ for some multiplicative constant $c_n$---for example, $\mathcal{L}(\theta) = \frac{1}{n}L(F(\theta),y)$. 
Suppose that a sequence of priors $\pi_d$ is chosen such that the corresponding prior duals $\pi
_d^\star$ take the form $c_d^\star e^{-d \mathcal{R}(z)}$ (that is, the prior duals concentrate as dimension increases). By \eqref{eq:DualPriorLinear}, this is true in the case of linear regression for zero-mean Gaussian priors with covariance $d^{-1}I$. 
Then, for fixed $\gamma > 0$, 
\[
\mathcal{Z}_{n}=c_{n}\int_{\mathbb{R}^{d}}e^{-n\mathcal{L}(\theta)}\pi_d(\theta)\dd \theta =  c_d^\star\int_{\mathbb{R}^{n}}e^{-d\mathcal{R}(z)}p^\star(y\vert z)\dd z = \mathcal{Z}_{n}^{\star}.
\]
That is, for an appropriate sequence of priors, in the overparameterized regime, the roles of dimension and sample size interchange. 
This suggests the possibility that for any class of overparameterized models, there exist corresponding priors such that marginal likelihood \emph{increases} with dimension, counter to the expected behavior, and resembling double descent phenomena.

\section{Numerical Experiments}
\label{sec:Numerics}

Although the IIC is derived from the marginal likelihood, it is unclear whether it correlates with other measures of out-of-sample prediction error; namely, the mean-squared error (MSE) over a test set. In this section, we consider several numerical examples to illustrate the efficacy of the IIC as a measure of model quality by comparing its behavior to more familiar metrics. 

Our experiments begin with generalized linear models, where the IIC and MSE are computed for the canonical example of overparameterized linear regression. Gamma regression is also considered to compare the IIC to evaluating a more complex (test) log-likelihood. 
We then investigate the behavior of the IIC for a family of models that are known to perform poorly in the interpolating regime: polynomial regression. In this setup, the IIC does \emph{not} exhibit the second descent curve, predicting that performance should not increase with overparameterization, as expected.
Finally, with an eye toward machine learning applications, we examine the recently introduced \emph{diagonal linear neural networks}, demonstrating that the IIC correlates with (an appropriately rescaled) MSE.

\subsection{Linear and Gamma Regression}
\label{sec:LinGamReg}

\begin{figure}
\centering
\includegraphics[width=0.38\textwidth]{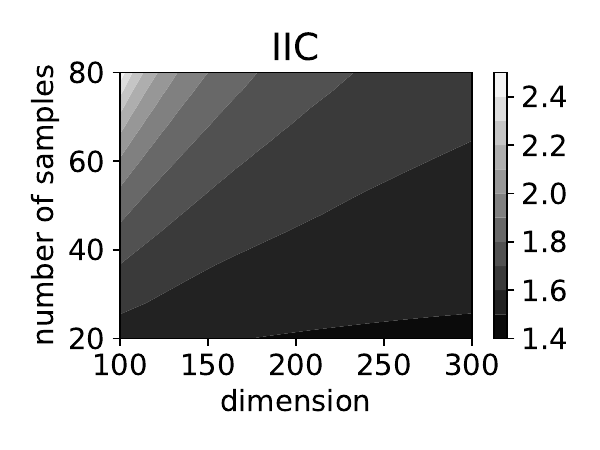}\qquad\qquad
\includegraphics[width=0.38\textwidth]{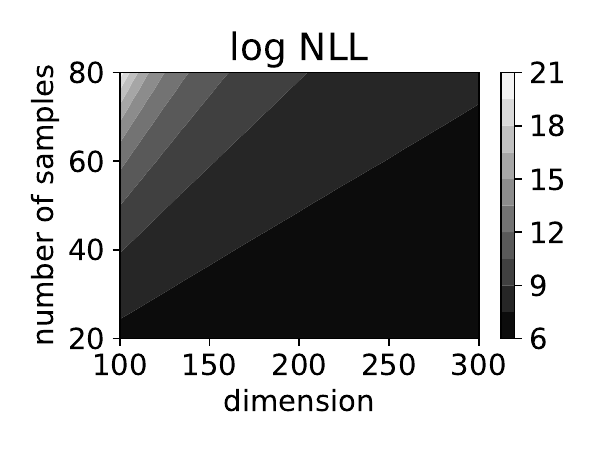}
\caption{\label{fig:GammaReg} IIC (left) and log test negative log-likelihood (right) for gamma regression, averaged over 50,000 replications on synthetic data. The IIC strongly correlates with the negative log-likelihood (Spearman correlation across 256 uniformly sampled points is $r = 0.996$).}
\end{figure}

Our first set of numerical experiments consider the behavior of the IIC in the Bayesian linear regression scenario with zero-mean Gaussian prior. As outlined in the preceding examples, in this setting, the interpolator of interest is the Moore--Penrose pseudoinverse solution, and the IIC is given explicitly by (\ref{eq:IICLinearReg}). To generate attributes, consider a random Fourier feature (RFF) model \cite{rahimi2007rff} fitted to a reduced $\mathrm{MNIST}$ dataset comprising a random sample of 1000 hand-written images. This closely follows the now standard experimental setup from \cite{belkin2019reconciling}. Here, predictors adopt the form $f(x,\theta) = \sum_{j=1}^d \theta_j \cos(z_j \cdot x)$, with each random feature map $\cos(z_j \cdot x)$ characterized by sampling $z_j \iidsim \mathcal{N}(0,1)$ randomly \emph{a priori}, then fixed during both training and inference. 
To examine the behaviour of the IIC as the number of attributes varies, 3000 features are generated, of which between $1$ and $3000$ are included as attributes. Note that the model is able to achieve interpolation at $d=n=1000$. 
Figure~\ref{fig:DDRFF} (top), presented earlier in the introduction, depicts the mean-squared error on the standard \texttt{MNIST} test set of 10,000 images (targeting the class label as a one-hot vector), clearly demonstrating the expected double-descent phenomenon. 
To compute the IIC, the parameters $\theta^\star$ are computed via the exact solution (\ref{eq:ridgeless}) and inserted into (\ref{eq:IICLinearReg}). 
Similarly, in its domain of definition when $d>n$, the IIC exhibits monotone decreasing behavior, capturing  the right half of the double-descent curve (orange) in Figure~\ref{fig:DDRFF} (bottom). 
Values of the BIC for linear regression are plotted for $d<n$ (blue), giving the left-half of the curve, outside the critical region $750<d<1000$. 
To complete the picture, values of the BIC (denoted BIC$-\lambda$) for ridge regression with parameter $\lambda=0.1$ are plotted (green dots), showing monotone increasing behavior after its minimum value.

In order to consider the performance of the IIC in the setting of {\em generalized} linear models, we consider the case of gamma regression.
Let $r > 1$ be an arbitrary concentration parameter of the response variable (this is distinct from $\gamma$). By letting
\[
\ell(v, \lambda) = \lambda\left(v - \frac{r-1}{\lambda}\right) - (r-1)\left(\log v - \log \frac{r-1}{\lambda}\right),
\]
and since $\frac{v}{x} - 1 \geq \log(\frac{v}{x})$, letting $x = \frac{r-1}{\lambda}$, we find that $\ell \geq 0$ and $\ell = 0$ if and only if $v = \frac{r-1}{\lambda}$. Verifying that $\int_0^\infty e^{-\ell(v,\lambda)} \dd v = \frac{1}{\lambda (r-1)^{r-1}} \Gamma(r)$, the corresponding Gibbs likelihood for $\ell$ with $\gamma = 1$ is the Gamma density
\begin{equation}
\label{eq:GammaDistn}
p_{\gamma = 1}(v \vert \lambda) = \frac{\lambda^r}{\Gamma(r)} v^{r-1} e^{-\lambda v}.
\end{equation}
Note that by taking $\gamma \to 0^+$, $p_\gamma$ will concentrate around the mode of the Gamma distribution. In our framework, we assert that the loss $\ell$ is minimized when its two arguments are equal. This property can be achieved by a reparameterization: taking $v = \exp y$, and letting $\lambda = (r-1) \exp(-f(x,\theta))$, we write
\begin{equation}
\label{eq:GammaNLL}
\ell(f(x,\theta), y) = (r-1)(\exp(y-f(x,\theta)) - 1 - (y - f(x,\theta))).
\end{equation}
The concentration parameter is now absorbed by $\gamma$ in the Gibbs likelihood. Furthermore, noting that $e^x \sim 1 + x + \tfrac12 x^2$ as $x \to 0^+$, the behavior of $\ell$ in the neighborhood of the interpolator $f(x,\theta) = y$ is identical to the mean-squared error. If $f(x,\theta) = x \cdot \theta$ is linear, the IIC for this regression problem is equivalent to (\ref{eq:IICLinearReg}) for linear regression where $y_i = \log v_i$, the logarithm of the labels.

Although the IIC can be readily computed in this context, it is entirely unclear whether it should align with evaluations of the test loss. This is especially true given that the IIC for both gamma and linear regression is identical in view of Remark~\ref{rem:lossequal}, and yet exhibit two different forms of likelihood.

To assess the performance of the IIC as a predictor of generalization, we assess the correlation between the IIC and the negative log-likelihood (NLL) of the gamma regression model (\ref{eq:GammaNLL}). 
Synthetic data was generated as follows: for fixed integer $p$, let $\theta_0 \sim \mathcal{N}(0, I_p)$ denote the set of target coefficients. For $i=1,\dots,n$, we sample base inputs $x_i^0 \iidsim \mathcal{N}(0, \frac14 I_p)$ and corresponding outputs $y_i \sim \mathrm{Gamma}(r, re^{-x_i^0 \cdot \theta})$ where $\mathrm{Gamma}(r,\lambda)$ denotes the Gamma distribution with density (\ref{eq:GammaDistn}). To augment the number of features, we create $d$ RFF features from each $x_i^0$ to obtain each $x_i \in \mathbb{R}^d$. In Figure~\ref{fig:GammaReg}, we let $p = 10$, the number of samples $n$ range from 20 to 80, the feature dimension $d$ range from 100 to 300, and plot the IIC and log test negative log-likelihood averaged over 50,000 independent data generations. Despite the different choice of loss, the IIC still shows a strong correlation with model quality as measured by the negative log-likelihood.

\subsection{Polynomial Regression}
\label{sec:Poly}

\begin{table}
\bgroup
\def\arraystretch{1.25}
\begin{tabular}{ccccc}
\toprule
\textbf{Nodes} & \multicolumn{2}{c}{\textbf{Equally-spaced}} & \multicolumn{2}{c}{\textbf{Chebyshev}}\\
& \multicolumn{2}{c}{$x_i = 1 - \frac{2(i-1)}{n-1}$} & \multicolumn{2}{c}{$x_i = \cos(\pi \cdot \frac{2i-1}{2n})$} \\ \midrule
 $\boldsymbol{n}$ & \textbf{IIC} & $\boldsymbol{\log}$ \textbf{MSE} & \textbf{IIC} & $\boldsymbol{\log}$ \textbf{MSE} \\\midrule
5
& $-0.447$ ($0.018$)
& $-1.561$ ($0.020$)
& $-1.011$ ($0.018$)
& $-2.130$ ($0.016$) \\
7
& $1.354$ ($0.022$)
& $-0.067$ ($0.028$)
& $0.389$ ($0.022$)
& $-1.578$ ($0.022$) \\
9
& $3.130$ ($0.022$)
& $1.481$ ($0.033$)
& $1.984$ ($0.023$)
& $-1.044$ ($0.026$) \\
11
& $4.917$ ($0.022$)
& $3.065$ ($0.037$)
& $3.588$ ($0.024$)
& $-0.565$ ($0.029$) \\
13
& $6.627$ ($0.022$)
& $4.634$ ($0.041$)
& $5.150$ ($0.025$)
& $-0.136$ ($0.032$) \\
15
& $8.314$ ($0.022$)
& $6.195$ ($0.045$)
& $6.736$ ($0.025$)
& $0.326$ ($0.035$) \\

\bottomrule
\end{tabular}
\egroup
\vspace{.5cm}
\caption{\label{tab:Poly}
Performance metrics (IIC and log MSE) for polynomial approximation with degree $d = n + 1$ to Runge's function $f(x) = (1 + 25 x^2)^{-1}$ under equally-spaced and Chebyshev node selections in the presence of added Gaussian noise $\epsilon\iidsim\mathcal{N}(0,0.3^2)$ to both inputs and outputs. Means and standard deviations are reported over 1000 replications.}
\end{table}

\begin{figure}[h]
    \includegraphics[width=0.7\textwidth]{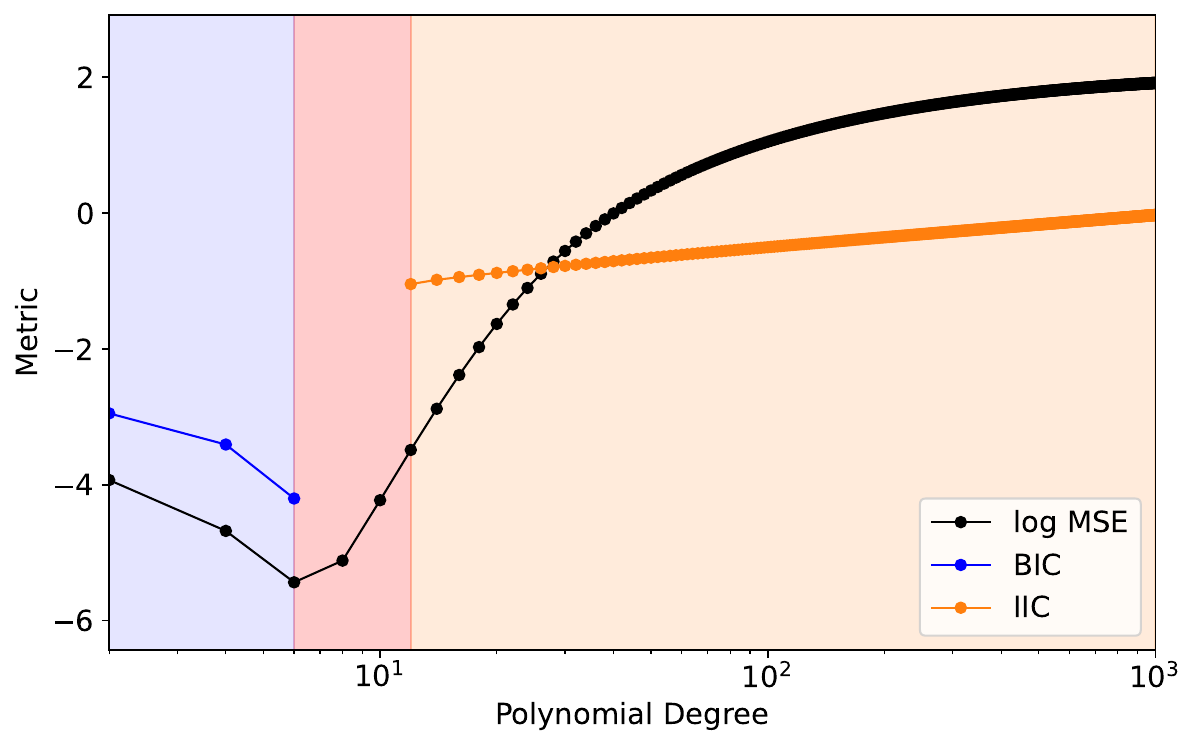}
\caption{\label{fig:PolyDD}Mean squared error (MSE; black) vs. classical BIC (\ref{eq:BICLinearReg}; blue), and our novel IIC (\ref{eq:IICLinearReg}; orange), for polynomial regression with increasing degree fitting to Runge's function $f(x) = (1 + 25 x^2)^{-1}$ over 10 equally spaced nodes on $[-1,1]$. \textbf{Polynomial regression does not exhibit double descent.}}
\end{figure}
Polynomial interpolators are among the most prominent examples to highlight that na\"{i}ve interpolation can result in poor quality fits. Hence, any measure of model quality in the overparameterized regime must be capable of predicting when such interpolators fail. For simplicity, we restrict ourselves to the scalar case: finding a function $f:\mathbb{R} \to \mathbb{R}$ that interpolates between $(x_1,y_1),\dots,(x_n,y_n)$ where each $x_i,y_i \in \mathbb{R}$ and $x_i \neq x_j$ for any $i \neq j$. As polynomial interpolators of this kind fall under the linear regression framework with input matrix $X = (\phi_j(x_i))_{i,j=1}^{n,d}$ derived from nodes $x_i$ by polynomial bases $\phi_j$, the IIC can easily be computed using (\ref{eq:IICLinearReg}). In perhaps the most famous case, where $\phi_j(x) = x^j$ are the monomials, $X$ is the \emph{Vandermonde matrix}.
Letting $\cond(X) = \frac{\sigma_{\max}(X)}{\sigma_{\min}(X)}$ denote the condition number of $X$, with $\sigma_{\max}$ and $\sigma_{\min}$ the largest and smallest singular values, respectively, since $\frac1n\log \sigma_{\max}(X)+\frac{n-1}{n}\log\sigma_{\min}(X) \leq \frac1{2n} \log \det(XX^\top) \leq \log \sigma_{\max}(X)$ and $\sup_{\|y\|=1}\|X^+ y\| = 1/\sigma_{\min}(X)$,
\begin{equation}
\label{eq:IICCond}
\frac{2}{n}\log \cond(X) - \log n \leq \sup_{\|y\| = 1} \IIC \leq 2 \log \cond(X) - \log n.
\end{equation}
While (\ref{eq:IICCond}) are coarse bounds in general, they provide a convenient interpretation for the behavior of the IIC in this case. Indeed, polynomial interpolators are commonly judged according to $\cond(X)$. The condition number is extensively well-studied for Vandermonde matrices when $n = d$ \cite{tyrtyshnikov1994bad}: if $x_1,\dots,x_n$ are real-valued nodes, then $X$ becomes notoriously ill-conditioned ($\cond(X) \geq 2^{n-2} n^{-1/2}$). From (\ref{eq:IICCond}), one should expect the IIC (and test error on noisy data) to become large whenever the monomial bases are used on an increasing number of nodes. Not all choices of nodes are equal, however; from \cite{gautschi1974norm}, if $x_i = 1 - \frac{2(i-1)}{n-1}$ are equally spaced on the interval $[-1,1]$, then $\cond(X) = \mathcal{O}(3.102^n)$, while the Chebyshev nodes $x_i = \cos(\pi \cdot \frac{2i-1}{2n})$ exhibit $\cond(X) = \mathcal{O}(2.415^n)$, indicating improved performance. Even in the case where $d=n+1$, this is realised for the IIC and log mean-squared error in Table \ref{tab:Poly}.

Polynomial regression also performs poorly with increasing model size. The Runge phenomenon \cite{epperson1987runge} demonstrates that increasing the degree can result in arbitrarily large mean-squared error. This is partially reflected in the condition number: if $x_i \in (-1,1)$, the squared spectral values of $X$ converge \emph{exponentially fast} as $d \to \infty$ to those of $H = ((1 - x_i x_j)^{-1})_{i,j=1}^n$. Even for moderate $n$, the condition number of $X$ does not improve substantially from the $n = d$ case. In Figure \ref{fig:PolyDD}, IIC correctly predicts that the double descent curve for random Fourier features (Figure \ref{fig:DDRFF}) is not expressed for polynomial interpolators.

\subsection{Diagonal Linear Neural Network}
\label{sec:DLNN}

\begin{figure}
\includegraphics[width=0.32\textwidth]{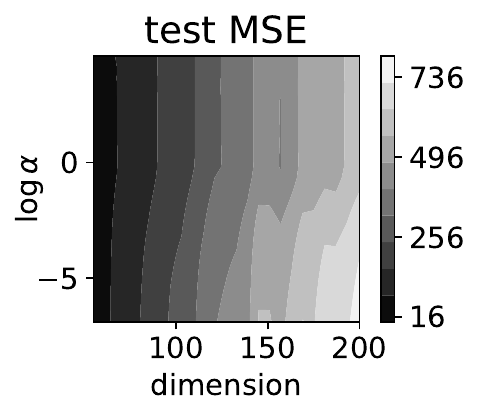}
\includegraphics[width=0.32\textwidth]{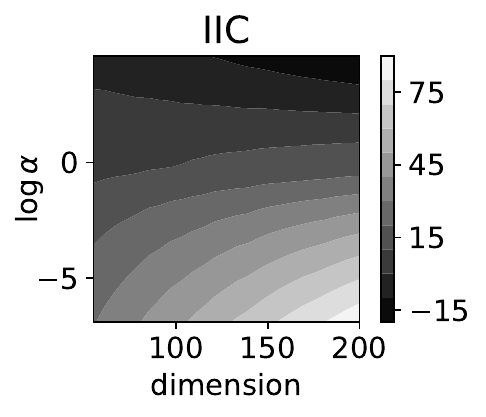}
\includegraphics[width=0.31\textwidth]{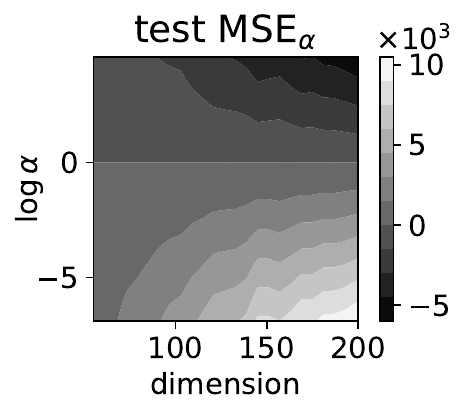}
\caption{\label{fig:DiagNN}IIC (left), MSE (center), and an adjusted MSE loss involving the initalisation parameter $\alpha$ (right) for diagonal linear neural networks. The IIC strongly correlates with the adjusted MSE$_\alpha$.}
\end{figure}

Our last example considers the diagonal linear neural network (DLNN) model of \cite{woodworth2020kernel}, which is noteworthy due to the fact that its implicit regularizer under gradient flow is both non-trivial, and has a closed form solution. In our framework, $\Theta = \mathbb{R}^{2d}$, $\mathcal{X} = \mathbb{R}^d$, $\mathcal{Y} = \mathbb{R}$, and
\begin{equation}
\label{eq:DiagLinNNModel}
f(x, \theta) = \sum_{j=1}^d (\theta_j^2 - \theta_{d+j}^2) x_j,\qquad \ell(y, y') = (y - y')^2.
\end{equation}
DLNNs are two-layer neural networks with quadratic activation function, where the final layer is comprised of $d$ copies of $\{-1,+1\}$. The model (\ref{eq:DiagLinNNModel}) exhibits the same hypothesis class as overparameterized linear regression with $d$ parameters, so $\mathcal{M}$ can be derived from (\ref{eq:LinSubspace}). But while gradient flow for overparameterized linear regression converges to the minimum norm solution (\ref{eq:ridgeless}), this is no longer true under the parameterization (\ref{eq:DiagLinNNModel}). In Lemma \ref{lem:ImpRegDiagNN}, we provide the precise expression of the regularizer under our framework. 
\begin{lemma}
\label{lem:ImpRegDiagNN}
Under Assumption \ref{ass:LocalReg}, let $d > \frac12 n$, and $\theta(t)$ satisfy gradient flow dynamics $\theta'(t) = -\nabla L(F(\theta(t)), y)$ with $\min_j |\theta_j(0)| > 0$. Then $\theta(t) \to \theta^\ast$ as $t\to\infty$, where $\theta^\ast$ satisfies (\ref{eq:INT}) with
\[
R(\theta)=\sum_{j=1}^{2d}\theta_{j}^{2}\left(2\log|\theta_{j}|-2\log|\theta_{j}(0)|-1\right)+ |\mathrm{sgn}(\theta_{j})-\mathrm{sgn}(\theta_{j}(0))|.
\]
Furthermore, the global minimizer $\argmin_{\theta \in \mathbb{R}^{2d}} R(\theta) = \theta(0)$.
\end{lemma}
The proof of Lemma \ref{lem:ImpRegDiagNN} is given in Appendix~\ref{sec:implicit}, following similar arguments to Theorem~1 in \cite{woodworth2020kernel}. For demonstration purposes, it is common to set $\theta_j(0) = \alpha$ for each $j$ and let $\alpha > 0$ vary. 
Using Lemma \ref{lem:ImpRegDiagNN}, we computed the IIC for the diagonal linear neural network model on a random subsample of 50 data points from the Energy Efficiency dataset \cite{tsanas2012accurate}, predicting only the heating load variable. Similarly, RFF is also used for feature generation, with the underlying feature dimension $d$ ranging between 55 to 200. Experiments were repeated 100 times at each configuration in order to account for the variation due to data subsampling and feature variation. Outcomes were averaged, and compared to the corresponding test MSE. The results are shown in Figure~\ref{fig:DiagNN}. Note that while the IIC does \emph{not} correlate with the MSE (the Spearman correlation coefficient between the IIC and MSE values was $r = 0.35$), it \emph{does} correlate closely with a rescaled version of the MSE. Noting that the log-prior $R(\theta)$ satisfies $R(\theta) \sim \log(\alpha^{-2})\|\theta\|_1$ as $\alpha \to 0^+$, we consider instead a rescaled version of the mean-squared error which exhibits similar behavior as $\alpha \to 0^+$, given by
\[
\text{MSE}_\alpha = \log(\alpha^{-2})\text{MSE}.
\]

After appropriately accounting for the behavior of the prior, the Spearman correlation coefficient increases significantly: between the IIC and $\text{MSE}_\alpha$ values, $r = 0.98$. This indicates that while the IIC can be a concrete estimate of model quality, nontrivial behavior in the regularization must first be carefully considered.

\section*{Acknowledgments}
LH is supported by the Australian Research Council through a Discovery Early Career Researcher Award (DE240100144).
FR was partially supported by the Australian Research Council through an Industrial Transformation Training Centre for Information Resilience (IC200100022). 
MWM would like to acknowledge the DOE, IARPA, NSF, and ONR for providing partial support of this work.

\clearpage
\appendix

\begin{center}
\LARGE \textbf{Appendix}
\end{center}

\section{Area and Coarea Formulae}\label{sec:gmt}

Duality in the marginal likelihood (Theorem \ref{thm:DR}) arises from the \emph{area} and \emph{coarea formulas} (particularly the latter), which decomposes a single integral into integrals over the level sets of a chosen function. These formulae are cornerstone results in geometric measure theory, with the general results for Lipschitz functions established by Federer \cite{federer1959curvature}. Our presentation of the result differs from most texts, but will be particularly convenient for our purposes. First, we present the area formula, as seen in \cite[Theorem 3.2.3]{federer1959curvature}.

\begin{theorem}[\textsc{Area Formula} \cite{federer1959curvature}]\label{thm:area}
Let $\Theta \subseteq \mathbb{R}^d$ be open, $f \in C_b(\mathbb{R}^n)$, and suppose that $F:\Theta \to \mathbb{R}^n$ is a real-valued locally Lipschitz function, \textbf{where} $\boldsymbol{d \leq n}$. For $\theta \in \Theta$, let $\boldsymbol{G(\theta) = DF(\theta)^\top DF(\theta)}$. Let $g$ be a measurable function such that $\theta \mapsto f(F(\theta))g(\theta) \det(G(\theta))^{-1/2} \in L^1(\Theta)$. Then
\[
\int_{\Theta} f(F(\theta)) g(\theta) \dd \theta = \int_{\mathbb{R}^n} f(z) \left(\sum_{\theta \in F^{-1}(z)} \frac{g(\theta)}{\det(G(\theta))^{1/2}} \right) \dd z.
\]
\end{theorem}

Of particular note is the assumption that $d \leq n$ (bolded), so while the area formula will be useful in some of the results below, it is less critical for establishing duality as the coarea formula presented below, as seen in \cite[Theorem 3.2.12]{federer1959curvature}.

\begin{theorem}[\textsc{Coarea Formula} \cite{federer1959curvature}]\label{thm:coarea}
Let $\Theta \subseteq \mathbb{R}^d$ be open, $f \in C_b(\mathbb{R}^n)$, and suppose that $F:\Theta \to \mathbb{R}^n$ is a real-valued locally Lipschitz function, \textbf{where} $\boldsymbol{d > n}$. For $\theta \in \Theta$, let $\boldsymbol{J(\theta) = DF(\theta) DF(\theta)^\top}$. Let $g$ be a measurable function satisfying $\theta \mapsto f(F(\theta)) g(\theta) \det(J(\theta))^{-1/2} \in L^1(\Theta)$. Then
\[
\int_{\Theta} f(F(\theta)) g(\theta) \dd \theta = \int_{\mathbb{R}^n} f(z) \left(\int_{F^{-1}(z)} \frac{g(\theta)}{\det(J(\theta))^{1/2}} \dd \mathcal{H}^{d-n}(\theta) \right) \dd z.
\]
\end{theorem}

Aside from its value in the proof of Theorem \ref{thm:DR}, the coarea formula is often useful for integrating under spherical coordinates.%
Consider $f \equiv 1$, $F(\theta) = \|\theta\|$, and $g \in L^1(\mathbb{R}^d)$. Then, for $S_d = \{x\,:\,\|x\|=1\}$ the $d$-dimensional sphere,
\[
\int_{\mathbb{R}^d} g(\theta) \dd \theta = \int_0^\infty \left(\int_{r S_d} g(\theta) \dd \mathcal{H}^{d-1}(\theta)\right) \dd r,
\]
and in particular,
\begin{equation}
\label{eq:CoareaRadial}
\int_{\mathbb{R}^d} g(\|\theta\|) \dd \theta = \frac{2\pi^{d/2}}{\Gamma(d/2)} \int_0^\infty r^{d-1} g(r) \dd r.
\end{equation}

\section{Integrating on a Submanifold}

The following lemma shows that integrating with respect to the Riemannian volume form is equivalent to integrating over the corresponding Hausdorff measure. 

\begin{lemma}\label{lem:}
Let $\mathcal{M}$ be an $m$-dimensional submanifold in $\mathbb{R}^d$, and let $f$ be a continuous real-valued function on $\mathcal{M}$. Then
\[
\int_{\mathcal{M}} f \dd V = \int_{\mathcal{M}} f(\theta) \dd \mathcal{H}^m(\theta).
\]
\end{lemma}
\begin{proof}
Let $\{(\mathcal{U}_\alpha,\varphi_\alpha)\}$ be a coordinate chart for $\mathcal{M}$. By \cite[Theorem 2.23]{lee2012smooth}, there exists a smooth partition of unity $\{\psi_\alpha\}$ of $\mathcal{M}$ over the sets $\{\mathcal{U}_\alpha\}$. Letting $G_\alpha(\theta) = D\varphi^{-1}_\alpha(\theta)^\top D\varphi^{-1}_\alpha(\theta)$ denote the matrix representation of the Riemannian metric tensor on $\mathcal{U}_\alpha$, by \cite[Proposition 15.31]{lee2012smooth},
\[
\int_\mathcal{M} f  \dd V = \sum_\alpha \int_{\varphi_\alpha(\mathcal{U}_\alpha)} \psi_\alpha(\varphi_\alpha^{-1}(\theta)) f(\varphi_\alpha^{-1}(\theta)) \sqrt{\det G_\alpha(\theta)} \dd \theta.
\]
On the other hand, by the Area Formula,
\[
\int_{\varphi_\alpha(\mathcal{U}_\alpha)} \psi_\alpha(\varphi_\alpha^{-1}(\theta))f(\varphi_\alpha^{-1}(\theta)) \sqrt{\det G_\alpha(\theta)} \dd \theta = \int_{\mathcal{U}_\alpha} \left(\sum_{\theta \in \varphi_\alpha(z)} \psi_\alpha(\theta) f(\theta)\right) \dd \mathcal{H}^m(z),
\]
and since,
\[
\sum_\alpha \int_{\mathcal{U}_\alpha} \left(\sum_{\theta \in \varphi_\alpha(z)} \psi_\alpha(\theta) f(\theta)\right) d \mathcal{H}^m(z) = \sum_\alpha \int_{\varphi_\alpha(\mathcal{U}_\alpha)} \psi_\alpha(\theta) f(\theta) \dd \mathcal{H}^m(\theta),
\]
and $\sum_\alpha \psi_\alpha(\theta) = 1$, the result follows.
\end{proof}

Let $d > n$ and consider the submanifold $\mathcal{M} = F^{-1}(z)$ where $F:\mathbb{R}^d \to \mathbb{R}^n$ and $z \in \mathbb{R}^n$. For any $p \in \mathcal{M}$, the orthogonal projection matrix $\Pi(p) \in \mathbb{R}^{d\times d}$ mapping vectors from $\mathbb{R}^d$ into the tangent space $T_p\mathcal{M}$ is given by
\[
\Pi(p) = I - DF(p)^\top J(p)^{-1} DF(p).
\]
The shape operator $S_p:\mathbb{R}^d \times \mathbb{R}^d \to \mathbb{R}^d$, also called the \emph{Weingarten map}, is given by
\[
S_p(u,w) = \nabla_u \Pi(p)[w],\qquad p \in \mathcal{M},\quad u \in \text{range}(\Pi(p)), \quad w \in \text{range}(I - \Pi(p)),%
\]
where $\nabla_z$ is the directional derivative in the direction of $z \in \mathbb{R}^d$, and $\nabla_z \Pi = (\nabla_z \Pi_{ij})_{ij}$. Note that since $\Pi(p) \nabla_u \Pi(p) = \nabla_u \Pi(p) (I - \Pi(p))$ (which one can find by differentiating $\Pi(p) = \Pi(p)^2$), this definition matches that of \cite[Theorem 1]{absil2013extrinsic}.
For a function $\eta\,:\,\mathbb{R}^d \to \mathbb{R}$ that is twice-differentiable, the Hessian for $\eta$ on $\mathcal{M}$ satisfies \cite[eq. 10]{absil2013extrinsic}
\[
\mathrm{Hess}_p \eta [u] = \Pi(p)\nabla^2 \eta(p) u + S_p(u, (I - \Pi(p))\nabla \eta(p)),\qquad p \in \mathcal{M},\quad u \in \text{range}(\Pi(p)).
\]
Assuming that $DF(p)$ is full-rank, the compact singular value decomposition of $\Pi(p)$ is $U_p U_p^\top$, where $U_p \in \mathbb{R}^{d\times (d-n)}$ is the semi-orthogonal matrix mapping $\mathbb{R}^{d-n}$ into $T_p\mathcal{M}$.
Let $\mathrm{Exp}_p\,:\,T_p \mathcal{M} \to \mathcal{M}$ denote the exponential map around a neighbourhood at $p \in \mathcal{M}$. The change of coordinates $\phi_p:\mathbb{R}^{d-n} \to \mathcal{M}$ given by
\[
\phi_p(w) = \mathrm{Exp}_p(U_p w)
\]
is a diffeomorphism between a neighbourhood $U\subset \mathbb{R}^{d-n}$ of zero and $V \subset \mathcal{M}$ of $p$. Since \linebreak $\left.\frac{\dd}{\dd t}\text{Exp}_{p}(tv)\right|_{t=0}=v$ for any $v \in T_p \mathcal{M}$, $\left.\frac{\dd}{\dd t}\phi_p(tw)\right|_{t=0}=U_p w$ for any $w \in \mathbb{R}^{d-n}$, and so $D \phi_p(0) = U_p$. Letting $G(w) = D\phi_p(w)^\top D\phi_p(w)$, $G(0) = I$. Since $\mathrm{Hess}_p \eta$ is projected onto $T_p \mathcal{M}$, it has rank at most $d-n$. Therefore, we consider
\[
\nabla_{\mathcal{M}}^2 \eta(p) = U_p^\top (\mathrm{Hess}_p \eta) U_p,
\]
which can be full rank. Furthermore, by \cite[Proposition 5.45]{boumal2023introduction},
\[
\nabla_{\mathcal{M}}^2 \eta(p) = \nabla^2 (\eta \circ \phi_p)(0).
\]

\section{Fibration Theorems}
\label{sec:Fibration}
The last of the three tools we require is a \emph{fibration theorem} to establish regularity of functions of the form:
\begin{equation}
\label{eq:RhoForm}
\rho(z) = \int_{F^{-1}(z)} g(\theta) \dd \mathcal{H}^m(\theta). 
\end{equation}
The fundamental challenge with differentiating an integral of the form (\ref{eq:RhoForm}) is the differentiation under the integral sign over the sets $F^{-1}(z)$. In the sequel, we will refer to $F^{-1}(z)$ as the \emph{fiber} of $z$ under $F$. Typically, to differentiate under the integral sign, one would employ the Leibniz integral rule.
\begin{theorem}[\textsc{Leibniz Integral Rule} \cite{flanders1973differentiation}]
\label{thm:Leibniz}
Let $\Omega(t)$ be a family of manifolds parameterized by a smooth flow $\varphi$, that is, for some $t_0$, $\varphi(t,\Omega(t_0)) = \Omega(t)$. Then
\begin{align}\label{eq:lappo}
\frac{\partial}{\partial t} \int_{\Omega(t)} F(t,x) \dd V = \int_{\Omega(t)} \frac{\partial F}{\partial t}(t,x) + \nabla_x \cdot \left(F(t,x) \frac{\partial}{\partial t} \varphi(t, x)\right) \dd V.
\end{align}
\end{theorem}
We remark that this allows us to take as many derivatives as $F$ and $\phi$ allow, in particular, if both $F$ and $\phi$ are $\mathcal{C^\infty}$-smooth, then so is \eqref{eq:lappo}.

Unfortunately, the conditions for the Leibniz integral rule (the existence of $\varphi$) are not so easily verified for (\ref{eq:RhoForm}), as it is not clear that there exists a smooth flow $\varphi(z,x)$ parameterizing $F^{-1}(z)$ with respect to a reference $F^{-1}(z_0)$, that is, $\varphi(z,F^{-1}(z_0)) = F^{-1}(z)$ for $z,z_0 \in \mathbb{R}^n$. Fortunately, the existence of such a flow $\varphi$ is the subject of significant prior work into \emph{fiber bundles}. 

\begin{definition}
A map $F:\mathbb{R}^d \to \mathbb{R}^n$ induces a \emph{smooth fiber bundle} if for each $z_0 \in F(\mathbb{R}^d) \subseteq \mathbb{R}^n$, there is a neighbourhood $N_0 \subseteq \mathbb{R}^n$ of $z_0$ and a smooth map $\varphi:N_0 \times \mathbb{R}^d$ such that
\[
\varphi(z, F^{-1}(z_0)) = F^{-1}(z),\qquad\text{for all }z \in N_0.
\]
\end{definition}

If $F$ induces a smooth fiber bundle, then the conditions of the Reynolds transport theorem are satisfied, and hence, we can differentiate $\rho$ as many times as we like. Now, assume for the moment that $F$ induces a smooth fiber bundle. Then for any $y \in F^{-1}(z_0)$,
\[
F(\varphi(z,y)) = z\quad\mbox{and so}\quad DF(\varphi(z,y))D\varphi(z,y) = I.
\]
Therefore, $D\varphi(z,y)$ must be a right inverse of $DF(\varphi(z,y))$. For this to occur, we will require that $F$ be a \emph{submersion}.

\begin{definition}
A differentiable map $F:\mathbb{R}^d \to \mathbb{R}^n$ is a \emph{submersion} if the Jacobian matrices $DF(\theta)$ have right inverses for all $\theta \in \mathbb{R}^d$. 
\end{definition}

The next theorems, the first by Meigniez \cite{meigniez2002submersions}, and the second presented by Ehresmann \cite{ehresmann1950connexions}, provide powerful sufficient conditions for a map to be a smooth fiber bundle. These conditions provide the basis for the Meigniez and Ehresmann conditions (resp.) stated in the main body. 

\begin{theorem}[\textsc{Meigniez Fibration Theorem} \cite{meigniez2002submersions}]
\label{thm:Meigniez}
Let $F:\mathbb{R}^d \to \mathbb{R}^n$ be a smooth submersion and assume $d > n$. If $F^{-1}(z)$ is diffeomorphic to $\mathbb{R}^{d-n}$ for each $z \in F(\mathbb{R}^d)$, then $F$ induces a smooth fiber bundle. 
\end{theorem}

\begin{theorem}[\textsc{Ehresmann Fibration Theorem} \cite{ehresmann1950connexions}]
\label{thm:Ehresmann}
Let $F:M \to N$ be a smooth surjective submersion. If $F^{-1}(K)$ is compact for every compact set $K \subseteq M$, then $F$ induces a smooth fiber bundle. 
\end{theorem}
Note that the conditions of the Ehresmann Fibration Theorem are satisfied if $F:\mathbb{R}^d \to \mathbb{R}^n$ is a smooth submersion such that, for some constants $m,p,r > 0$,
\[
\|F(x)\| \geq m \|x\|^p \text{ for all } \|x\| > r.
\]
Indeed, this condition implies that for any $x \in F^{-1}(z)$,
$\|x\| \leq \min\{r, m^{-1/p} \|z\|^{1/p}\}$, 
and hence $F^{-1}(K)$ is compact for any compact set $K$.

\section{Laplace Approximation}
\label{sec:Laplace}

The final essential tool we require is an asymptotic expansion for integrals in the low temperature regime. The following can be found in \cite[Theorem 15.2.5]{simon2015advanced}.

\begin{lemma}[\textsc{Laplace Approximation} \cite{simon2015advanced}]
\label{lem:Laplace}

Let $\Theta \subseteq \mathbb{R}^d$ be open, and $\eta,g$ be real-valued functions with $g$ non-negative and $\eta$ attaining a unique global minimum at $x_0$. Assume that $\eta$ and $g$ are $\mathcal{C}^{\infty}$-smooth in a neighbourhood of $x_0$, $x_0 \in \Theta$, $g(x_0) \neq 0$, and $\nabla^2 \eta(x_0)$ is non-singular. Then as $\gamma \to 0^+$,
\[
\int_{\Theta} e^{-\frac1\gamma \eta(x)} g(x) \dd x = (2\pi\gamma)^{n/2} \det(\nabla^2 \eta(x_0))^{-1/2} e^{-\frac1\gamma \eta(x_0)} g(x_0)[1 + \mathcal{O}(\gamma)],
\]
provided the left-hand side is integrable for some $\gamma > 0$. 
\end{lemma}

The following lemma demonstrates that our conditions for global regularity (Meigniez and Ehresmann) imply local regularity (Assumption \ref{ass:LocalReg}). 
\begin{lemma}
\label{lem:Condition}
Under Assumptions \ref{ass:Base} and \ref{ass:UniquePrior}, the Meigniez and Ehresmann conditions are sufficient for Assumption \ref{ass:LocalReg}.
\end{lemma}
\begin{proof}
From Proposition \ref{thm:regularity}, it will suffice to show that $\limsup_{\gamma \to 0^+} \mathcal{Z}_{n,\gamma} < +\infty$ whenever $\pi^\star$ is smooth. By Proposition \ref{thm:DR}, $\mathcal{Z}_{n,\gamma} = \mathcal{Z}_{n,\gamma}^\star$ where
\[
\mathcal{Z}_{n,\gamma}^\star = \int_{\mathbb{R}^{mn}} p_{n,\gamma}^\star(y\vert z) \pi^\star(z) \dd z. 
\]
Following the steps in the proof of Theorem \ref{thm:IIC}, since $\pi^\star$ is smooth, $\mathcal{Z}_{n,\gamma} = \mathcal{Z}_{n,\gamma}^\star = \pi^\star(y)[1 + \mathcal{O}(\gamma)]$.
\end{proof}

\section{Implicit Regularization of Diagonal Linear Neural Network}
\label{sec:implicit}
We now present an adaptation of the derivation of the implicit regularizer presented in Theorem~1 of \cite{woodworth2020kernel} in terms of the weights themselves.

\begin{proof}[Proof of Lemma \ref{lem:ImpRegDiagNN}]
Letting $\circ$ denote the element-wise Hadamard product, $\boldsymbol{X}=(X,-X)$, and $r(t)=\boldsymbol{X}\theta(t)\circ\theta(t)-y,$ then
$$\theta'(t)=-\nabla\|\boldsymbol{X}\theta(t)\circ\theta(t)-y\|^{2}=-2(\boldsymbol{X}^{\top}r(t))\circ\theta(t).$$
For $w(t)=\theta(t)\circ\theta(t)$ so that $r(t)=\boldsymbol{X}w(t)-y$, $$w'(t)=2\theta'(t)\circ\theta(t)=-4(\boldsymbol{X}^{\top}r(t))\circ w(t),$$
which has a solution of the form
$$w(t)=w(0)\circ\exp\left(-4\boldsymbol{X}^{\top}\int_{0}^{t}r(s)\dd s\right),$$
where the exponential is to be interpreted element-wise. Taking $w^{\ast}=\lim_{t\to\infty}w(t)$ and assuming $w^{\ast}$ achieves a global minimum,$$\boldsymbol{X}w^{\ast} = y, \qquad w^{\ast}=B(\boldsymbol{X}^{\top}\nu),$$
where $B(z)=w(0)\circ\exp(z)$ and $\nu=-4\int_{0}^{\infty}r(s)\dd s$. The KKT optimality conditions for the problem $$w^{\ast}=\argmin_{w\in [0,\infty)^{2d}}Q(w)\text{ such that }\boldsymbol{X}w^{\ast}=y$$
are $\boldsymbol{X}w^{\ast}=y$ and $\nabla Q(w^{\ast})=\boldsymbol{X}^{\top}\nu$  for some multipliers $\nu$. One solution to this problem is $\nabla Q(z)=B^{-1}(z)$, where $B^{-1}$ is the element-wise inverse of $B$, as this implies $\nabla Q(w^{\ast})=\boldsymbol{X}^{\top}\nu$. Note that $B^{-1}(z)=\log(z / w(0))$, where the vector division is again performed element-wise. This implies that $$Q(z)=\sum_{j}z_{j}\left(\log\left(\frac{z_{j}}{w_{j}(0)}\right)-1\right)+C,\qquad\mbox{for some }C>0,$$
which we set to zero. The only critical point of $Q$ occurs at $z=w(0)$, and one can verify that it is a global minimum. This is the regularizer for $w$, which is not one-to-one with $\theta$. To form the implicit regularizer for $\theta$, observe that the gradient flow solution has the same sign as $\theta(0)$. Therefore, $$\theta^{\ast}=\argmin_{\theta \in \mathbb{R}^{2d}}[\tilde{Q}(\theta)+S(\theta)]\text{ subject to }f(\theta,X)=y,$$ where $S(z)=\sum_{j}|\text{sgn}(z_{j})-\text{sgn}(\theta_{j}(0))|$ and $$\tilde{Q}(z)=\sum_{j}z_{j}^{2}\left(2\log|z_{j}|-1-2\log|\theta_{j}(0)|\right).$$
The global minima of $\tilde{Q}$ are located at $(\pm\theta_{j}(0))_{j}$, so the global minimum of $\tilde{Q}+S$ is $\theta(0)$. 
\end{proof}

\end{document}